\documentclass{article}

\usepackage{amsmath,amsfonts,bm}

\def\eqref#1{equation~\ref{#1}}

\def\1{\bm{1}}

\def\vb{{\bm{b}}}

\def\ve{{\bm{e}}}

\def\vv{{\bm{v}}}

\def\vx{{\bm{x}}}

\def\mI{{\bm{I}}}

\def\mX{{\bm{X}}}
\def\mY{{\bm{Y}}}
\def\mZ{{\bm{Z}}}

\DeclareMathAlphabet{\mathsfit}{\encodingdefault}{\sfdefault}{m}{sl}
\SetMathAlphabet{\mathsfit}{bold}{\encodingdefault}{\sfdefault}{bx}{n}

\def\sA{{\mathbb{A}}}

\usepackage{microtype}
\usepackage{graphicx}
\usepackage{subcaption}
\usepackage{booktabs} 

\usepackage{hyperref}

\usepackage[preprint]{icml2026}

\usepackage{amsmath}
\usepackage{amssymb}
\usepackage{mathtools}
\usepackage{amsthm}

\usepackage[capitalize,noabbrev]{cleveref}

\theoremstyle{plain}
\newtheorem{theorem}{Theorem}[section]

\theoremstyle{definition}

\newtheorem{assumption}[theorem]{Assumption}
\theoremstyle{remark}

\usepackage[textsize=tiny]{todonotes}

\usepackage{xspace}
\usepackage{enumitem}
\usepackage{tcolorbox}
\usepackage{wrapfig}
\usepackage{booktabs, makecell, multirow, tabularx}
\usepackage{colortbl}
\usepackage{stfloats}

\newcommand{\op}[1]{\operatorname{#1}}
\newcommand{\themodel}{{COIN}\xspace}
\newcommand{\conr}{{Context Reliance}\xspace}

\icmltitlerunning{Uncovering Context Reliance in Unstructured Knowledge Editing}

\begin{document}

\twocolumn[
  \icmltitle{Uncovering Context Reliance in Unstructured Knowledge Editing}



  \icmlsetsymbol{equal}{*}

  \begin{icmlauthorlist}
    \icmlauthor{Zisheng Zhou}{equal,1}
    \icmlauthor{Mengqi Zhang}{equal,1}
    \icmlauthor{Shiguang Wu}{1}
    \icmlauthor{Xiaotian Ye}{2}
    \icmlauthor{Chi Zhang}{1}
    \icmlauthor{Zhumin Chen}{1}
    \icmlauthor{Pengjie Ren}{1}
  \end{icmlauthorlist}

  \icmlaffiliation{1}{Shandong University}
  \icmlaffiliation{2}{Beijing University of Posts and Telecommunications}

  \icmlcorrespondingauthor{Pengjie Ren}{renpengjie@sdu.edu.cn}

  \icmlkeywords{Machine Learning, Large Language Model, Knowledge Editing}

  \vskip 0.3in
]



\printAffiliationsAndNotice{\icmlEqualContribution}  

\begin{abstract}
Editing Large language models (LLMs) with real-world, unstructured knowledge is essential for correcting and updating their internal parametric knowledge.
In this work, we revisit the fundamental next-token prediction (NTP) as a candidate paradigm for unstructured editing.
We identify \textbf{\conr} as a critical failure mode of NTP-based approaches, where knowledge acquired from edited text becomes highly dependent on its preceding context, leading to recall failures when that context is absent during inference.
This hypothesis is supported by our \textit{empirical} validation that prepending context during inference recovers knowledge recall.
We further \textit{theoretically} demonstrate that \conr is an inherent consequence of gradient-based optimization, which tends to bind acquired knowledge to a specific aggregated contextual representation.
To address this, we propose a simple yet effective COntext-INdependent editing framework (\themodel), encouraging model to focus on knowledge within local scope rather than memorizing contextual patterns.
Evaluations show that \themodel reduces \conr by 45.2\% and outperforms strong baselines by 23.6\% in editing success rate, highlighting the vital role of mitigating \conr for robust editing.
\end{abstract}

\section{Introduction}

\begin{figure}
  \centering
  \includegraphics[width=1\columnwidth]{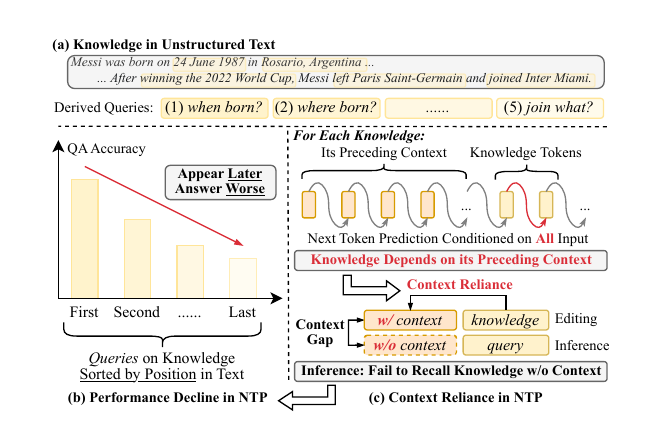}
  \vspace{-1.7em}
  \caption{(a) Unstructured text containing multiple pieces of knowledge. (b) Edited model tends to recall knowledge appearing later in the text less accurately. (c) Knowledge acquired from edited text is overly dependent on its preceding context, and omitting that context during inference leads to failed recall.}
  \label{fig:intro}
  \vspace{-1.3em}
\end{figure}

Large language models (LLMs) acquire extensive world knowledge during pre-training by next-token prediction (NTP) \citep{llmSurvey,unkeFT}, yet this knowledge is static and may contain factual errors or become outdated over time \citep{de2021editing}.
Knowledge editing \citep{KnowledgeEditing} has emerged as a promising solution, enabling efficient and precise modifications to specific knowledge within LLMs.
Existing work \citep{rome,glame,alphaedit} primarily focuses on structured knowledge editing, where knowledge is represented as (\textit{subject}, \textit{relation}, \textit{object}) triplets.
In contrast, the vast majority of knowledge in the real-world is expressed in unstructured text \citep{unstructuredKnowledge}, characterized by richer information as shown in Figure \ref{fig:intro} (a).
Despite this, most existing unstructured knowledge editing methods \citep{selfedit,unke,anyedit} largely adapt strategies originally designed for structured knowledge, limiting their ability to faithful editing.

Given that NTP is the fundamental paradigm through which LLMs naturally acquire knowledge during pre-training, it appears to be a theoretically suitable paradigm for unstructured knowledge editing \citep{unkeFT}.
However, its reliability in this specific task remains to be examined.
To investigate, we conduct an empirical study on several benchmarks (\S \ref{sec:exp_setup}).
Our analysis reveals a critical observation: after editing unstructured text containing multiple pieces of knowledge, the model tends to \textit{recall knowledge appearing later in the text less accurately}, as shown in Figure \ref{fig:intro} (b).
To explain this, we propose \textbf{\conr, where the knowledge acquired from the edited text is highly dependent on its preceding context} (\S \ref{sec:uncover_context_reliance}).
Consequently, when the model is queried without that specific context during inference, this context gap makes it difficult to accurately recall the relevant knowledge, as illustrated in Figure \ref{fig:intro} (c).
As the context length grows for later-appearing knowledge, this gap widens, resulting in the observed performance decline.
To illustrate this, consider a scenario where the model learns the fact ``\textit{Messi won the 2022 World Cup.}'' from a text that also includes the preceding context ``\textit{Messi was born on 24 June 1987 \dots},'' it may associate this fact with the context.
If this context is omitted when asking ``\textit{What did Messi win in 2022?}'', the model may fail to recall the fact.

We validate the \conr hypothesis through both empirical evidence and theoretical analysis.
\textit{Empirically}, we observe that prepending the original context during inference significantly recovers knowledge recall, particularly for facts appearing later in the text (Figure \ref{fig:append_prefix}).
This confirms that the edited model heavily relies on the preceding context as a necessary cue for accessing the learned knowledge.
\textit{Theoretically}, we provide a rigorous grounding for these observations by demonstrating that \conr is an inherent consequence of gradient-based optimization.
Theorem \ref{thm:main-context-dep} indicates that gradient updates tend to entangle target knowledge with a specific aggregated contextual representation, rather than exclusively with semantically relevant tokens.
Taken together, these findings reveal that while knowledge is successfully encoded, \conr renders it accessible only under specific textual conditions, highlighting this phenomenon as a fundamental limitation of standard NTP-based editing.

To mitigate \conr, we propose a simple yet effective COntext-INdependent unstructured knowledge editing framework (\themodel), designed to decouple knowledge association from its context.
Specifically, \themodel introduces the Context Alignment Loss, which enforces prediction consistency across varying context lengths, thereby forcing the model to concentrate on intrinsic knowledge within the local scope rather than memorizing contextual patterns.
Additionally, to prevent model collapse \citep{model_collapse}, we propose the Knowledge Consistency Loss to preserve the model’s capabilities on unrelated knowledge.
We conduct comprehensive evaluation utilizing Llama3-8B \citep{llama3} and Qwen2.5-7B \citep{qwen2.5} on AKEW the \citep{akew} and UnKEBench \citep{unke} benchmarks.
Experimental results confirm that \themodel effectively mitigates \conr by 45.2\%.
By alleviating this, the framework achieves a substantial 25.6\% improvement in editing success rate for unstructured tasks.
Furthermore, in structured knowledge scenarios, \themodel demonstrates superior generalization, outperforming all baselines in multi-hop reasoning tasks.
These findings underscore that mitigating \conr is pivotal for achieving the robustness and generalization required in real-world knowledge editing.

We summarize our contributions as follows:
\vspace{-0.8em}
\begin{itemize}[leftmargin=*]
  \setlength{\itemsep}{1.2pt}
    \item We identify \conr as a critical failure mode of NTP-based unstructured knowledge editing. Our key observation reveals that models tend to couple learned knowledge with its preceding context, leading to recall failures when such context is absent during inference.
    \item We validate \conr by empirical and theoretical analysis, and demonstrate that it is an inherent consequence of gradient-based optimization, driving models to rely on aggregated contextual representations.
    \item We propose \themodel, a simple yet effective framework for mitigating \conr. Extensive experiments demonstrate that \themodel reduces \conr by 45.2\% and improves the editing success rate by 25.6\%.
\end{itemize}

\section{Related Work}

\textbf{Structured Knowledge Editing} targets knowledge represented as well-defined (\textit{subject}, \textit{relation}, \textit{object}) triplets. These methods can be primarily classified into three categories: external memorization-based \citep{grace,mello}, meta-learning-based \citep{mend,malmen}, and locate-then-edit methods \citep{rome,alphaedit}.

\textbf{Unstructured Knowledge Editing} is designed for more realistic scenarios where knowledge involves complex contexts and nuanced semantic relationships within texts. These methods are typically categorized into two groups:
(1) \textit{Triplets decomposition-based methods} convert complex unstructured texts into triplets, subsequently applying standard structured knowledge editing techniques. For instance, \citet{selfedit} leverage the eventual context to decompose text into a series of sub-questions and corresponding answers to perform editing.
(2) \textit{Query construction-based methods} focus on constructing specific query for target text, editing the model based on the resulting QA pair. UnKE \citep{unke} aggregates contextual information from input query across multiple layers to inject knowledge into specific MLP modules. AnyEdit \citep{anyedit} segments unstructured text into blocks for iterative editing based on generated query. $\mu$KE \citep{muke} further enhances the process with a Matryoshka-style memory update mechanism and adaptive loss coefficients. Detailed related work is provided in Appendix \ref{appd:related_work}.

Despite these advancements, current approaches predominantly rely on transforming raw text into intermediate formats (e.g., triplets or QA pairs). This transformation process often oversimplifies complex semantics, leading to the loss of critical information.
In contrast, as the fundamental paradigm for native knowledge acquisition in LLMs, NTP paradigm directly utilizes raw text, theoretically avoiding such information loss \citep{triple_loss}. However, the application and behavior of the NTP paradigm in unstructured knowledge editing remain underexplored. To address this gap, we conduct a comprehensive investigation into unstructured knowledge editing based on NTP paradigm.

\begin{figure*}[ht]
  \centering
  \includegraphics[width=0.9\linewidth]{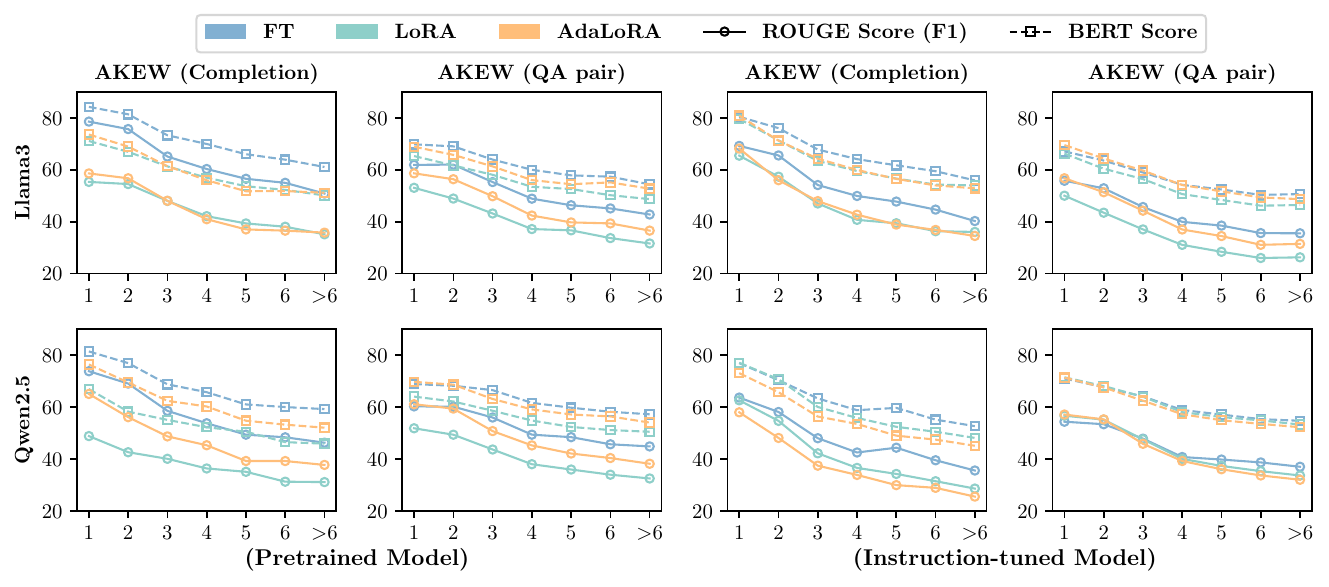}
  \vspace{-0.6em}
  \caption{Performance decline of different NTP-based methods on AKEW dataset as the position of knowledge moves later in the text. The x-axis represents the sequential order of knowledge within the original text, where position `1' corresponds to questions targeting knowledge positioned at the beginning, and position `$>$6' corresponds to knowledge located towards the end. The y-axis represents the corresponding accuracy.}
  \vspace{-0.7em}
  \label{fig:context_reliance}
\end{figure*}

\section{Uncovering \conr in NTP}\label{sec:analysis}

In this section, we first examine the NTP-based editing paradigm and observe that its effectiveness diminishes for knowledge appearing later in text. We attribute this phenomenon to \conr, the dependence of acquired knowledge on its preceding context, and provide both empirical and theoretical analysis. Finally, we demonstrate the insufficiency of existing techniques in addressing this issue.

\subsection{Experimental Setup}\label{sec:exp_setup}
\textbf{Models \& Methods.} We conduct our experiments using the pretrained and instruction-tuned versions of Llama3-8B and Qwen2.5-7B models.
We evaluate three training methods, including Fine-Tuning (FT) \citep{ft}, LoRA \citep{lora}, and AdaLoRA \citep{adalora}. All three methods employ NTP paradigm for training on the full unstructured text.
FT optimizes specific MLP module, while LoRA and AdaLoRA introduce low-rank adapters to the weight matrices of each MLP layer. Further implementation details are provided in Appendix \ref{appd:baselines}.

\textbf{Dataset.} To evaluate the performance of the NTP paradigm in editing unstructured knowledge, we utilize the AKEW dataset \citep{akew}.
Each instance in AKEW contains an unstructured text and several questions formatted as text completion tasks, where each question targets a specific piece of knowledge found within that text.
These questions are \textit{ordered sequentially} based on the location of the target knowledge in the unstructured text.
We apply the NTP editing method directly to the original unstructured text without any pre-processing.
To investigate the impact of query format, we also convert the original text completion tasks into a standard QA format.
Further details about the AKEW dataset are available in Appendix \ref{appd:akew_details}.

\textbf{Metrics.} Following prior works \citep{unke,anyedit}, we adopt BERT and ROUGE Scores as our primary metrics for a comprehensive evaluation. We employ BERT Score to evaluate the semantic similarity between the model-generated output and the ground truth, by computing the cosine similarity of their contextual word embeddings.
Additionally, we use ROUGE Scores for lexical evaluation, calculating the precision and recall of ROUGE-L (Longest Common Subsequence) and reporting the F1 score.
Notably, unlike \citet{unke} and \cite{anyedit}, we measure performance by questioning specific knowledge within the text, rather than by reproducing the entire text.

\subsection{Empirical Analysis: \conr Impairs Knowledge Recall}\label{sec:uncover_context_reliance}
\textbf{Observation: Edited models exhibit significant recall degradation for knowledge positioned later in the text.}
Figure \ref{fig:context_reliance} illustrates how the accuracy declines for questions that target knowledge located in different positions in the input text.
Specifically, for the knowledge appears later in the text, we observe a sharp decrease in model performance, with BERT and ROUGE scores dropping by up to 38.3\% and 56.0\%, respectively.
This degradation is consistent across various NTP-based methods, models, and query formats, indicating a systemic vulnerability in current editing paradigms.
Our experiments on the UnKEBench \cite{unke} dataset, as shown in Figure \ref{fig:unke_context_reliance}, further strongly support this finding.

\textbf{Hypothesis: NTP objectives induce \conr, making edited knowledge heavily dependent on its preceding context.}
We attribute the observed performance decline to \conr.
Specifically, during training, the model learns knowledge by \textit{maximizing the probability of predicting the next token conditioned on the entire preceding input sequence}, causing the model to deeply associate the acquired knowledge with this specific context.
However, during inference, input queries typically lack this preceding information, preventing the model from effectively recalling the edited knowledge.
This context gap intensifies for knowledge located later in the text, as the longer preceding context exerts a stronger influence, thereby explaining the observed positional degradation.

\begin{figure}[ht]
\centering
\vspace{-0.5em}
\includegraphics[width=0.9\linewidth]{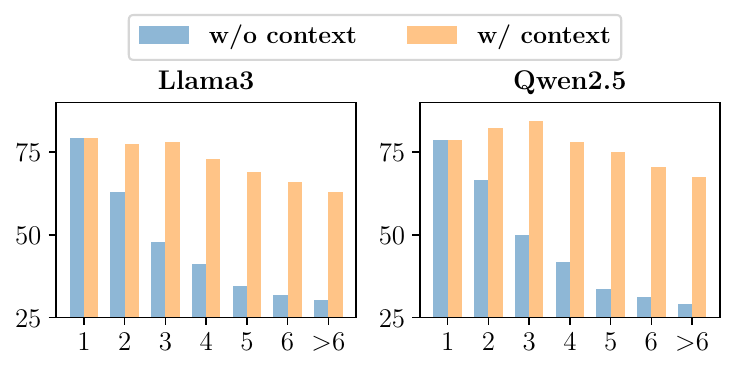}
\vspace{-0.5em}
\caption{Probability of predicting answer with and without preceding context. The results show that appending context increases probability of answer.}
\vspace{-0.3em}
\label{fig:append_prefix}
\end{figure}

\textbf{Validation: Prepending context during inference restores knowledge recall, confirming our \conr hypothesis.}
If \conr is the cause of the recall failure, then restoring the original preceding context at inference time should recover the model’s ability to answer correctly.
To verify this, we compare the probability of correct answer on standard queries with queries prepended by the original context from the editing text.
Importantly, we ensure that the correct answers do not appear in the prepended context to prevent trivial copying.
As shown in Figure \ref{fig:append_prefix}, when no context is provided, the probability of correctly predicting the answer drops significantly for knowledge appearing later in the text.
However, once the preceding context is restored, this probability is almost recovered.
This result shows that the edited knowledge indeed depends on the presence of the preceding context, providing strong evidence for the \conr hypothesis.

\subsection{Theoretical Analysis: Gradient-Based Optimization Induces \conr}\label{sec:theoretical_validation}
We provide a theoretical analysis explaining why the NTP paradigm can naturally induce \conr.
Our key observation is that \textbf{gradient-based optimization tends to bind newly acquired knowledge to the aggregated contextual representation produced during editing}, rather than associating it exclusively with the semantically relevant tokens.
As a result, the learned behavior may depend on auxiliary context cues that are present during training but absent at inference.
We formalize this intuition and show how gradient descent can induce such \conr.

\begin{assumption}[Informal]
\label{a:settings}
Consider a reparameterized transformer model under the setting of~\citet{tian2023scan}.
The model is trained via gradient descent on a single sample $(x_1,\dots,x_T,x_{T+1})$, where $x_i\in[M]$ and $M$ is the vocabulary size.
The training objective is the log-likelihood of predicting $x_{T+1}$ given the prefix $x_{\le T}$.
For analytical convenience while retaining the core mechanism of interest, we assume there is exactly one relevant and context token, denoted by $p$ and $q$, respectively, which together receive almost all attention from the query token $x_T$.
\end{assumption}

\begin{theorem}[\conr Induced by Gradient Descent]
\label{thm:main-context-dep}
Under the setting of Assumption~\ref{a:settings}, after one gradient update, the model exhibits the following behavior during inference, where $\op{logit}\in\mathbb{R}^M$ denotes the prediction logits.

\textbf{Success with context.}
If both contextual tokens $p$ and $q$ are present, then
\begin{equation}
    \arg\max_{k\in[M]} \op{logit}_k = x_{T+1}.
\end{equation}
\textbf{Failure without context.}
If the contextual token $q$ is absent, then
\begin{equation}
    \op{logit}_{x_{T+1}} < \max_{k\in\sA} \op{logit}_{k},
\end{equation}
where $\sA$ denotes the set of tokens associated with $p$ prior to training, and $x_{T+1}\notin\sA$.
\end{theorem}

The mechanism follows from how gradient-based optimization interacts with contextual representations in transformers.
The prediction head operates on an attention-weighted aggregation of contextual tokens, which during training is dominated by the joint contribution of $p$ and $q$.
As a result, the gradient update reinforces the association between the target token $x_{T+1}$ and this aggregated representation, rather than with the relevant token $p$ alone.
The model therefore predicts the target correctly only when the same contextual combination is present.
When the context token $q$ is removed at inference time, the aggregated representation shifts substantially, causing a mismatch with the learned update direction.
The model then reverts to its prior behavior, dominated by the pre-existing associations in $\sA$, and fails to recall the newly learned knowledge.
This explains why the edited knowledge appears context-dependent.
The complete technical statement and proof are given in Appendix~\ref{app:contextual-dependency}.

\subsection{Intuitive Mitigation Techniques}

To further investigate how intuitive strategies influence \conr, we conduct additional experiments analyzing existing techniques, including knowledge splitting and paraphrasing.
Further analysis on the impact of model scale and training steps on \conr phenomenon is provided in Appendix \ref{appd:model_scale} and \ref{appd:training_steps}.

\begin{figure}[ht]
\centering
\vspace{-0.5em}
\includegraphics[width=0.85\linewidth]{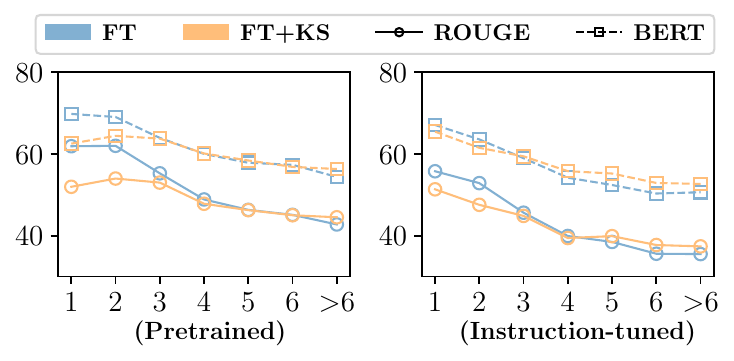}
\vspace{-0.5em}
\caption{Performance comparison between with and without knowledge splitting (KS).}
\vspace{-0.3em}
\label{fig:knowledge_splitting}
\end{figure}

\textbf{Mitigation Technique 1: Knowledge Splitting.}
Given that \conr arises from the strong correlation between knowledge and the preceding context, a natural mitigation strategy is to segment unstructured texts into individual sentences for training.
As shown in Figure \ref{fig:knowledge_splitting}, while this strategy slightly improves accuracy for later knowledge, it significantly degrades performance for knowledge at the beginning, and the overall decreasing trend remains.
We believe this happens because the simple splitting, while effective at breaking reliance on contexts, destroys critical logical connections between sentences.
For example, it can make model unable to figure out what ``he'' or ``it'' refers to in the text.
As a result, the model learns from incomplete or ambiguous information.

\begin{figure}[ht]
\centering
\vspace{-0.5em}
\includegraphics[width=0.85\linewidth]{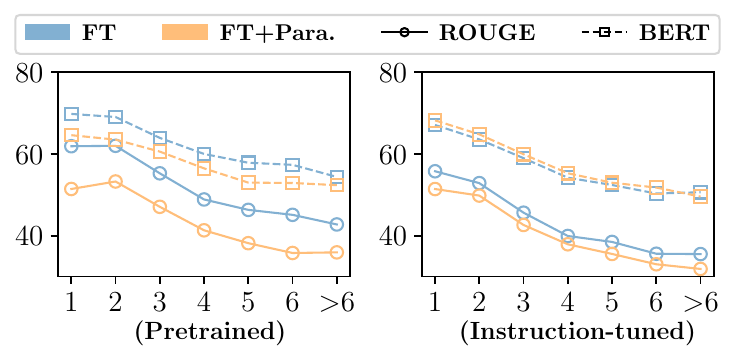}
\vspace{-0.5em}
\caption{Performance comparison between with and without paraphrasing (Para.).}
\vspace{-0.3em}
\label{fig:paraphrasing}
\end{figure}

\textbf{Mitigation Technique 2: Paraphrasing.}
Paraphrasing is a classic data augmentation technique, and we apply it to force model learning from diverse phrasings rather than memorizing fixed patterns.
To do this, we generate multiple versions of the source texts that are semantically identical but syntactically and stylistically varied.
However, results in Figure \ref{fig:paraphrasing} show that paraphrasing not only hurt knowledge acquisition across all positions but also failed to mitigate decline.
We attribute this failure to two potential factors.
First, the paraphrasing process may have introduced extra noise that interfered with the model's learning process.
Second, while the syntax varies, the preserved semantic consistency may have inadvertently reinforced the model's reliance on contextual cues.

The experiments above indicate that existing strategies cannot effectively mitigate the \conr issue in the NTP paradigm.
Therefore, developing new methods to address this problem is critical for advancing unstructured knowledge editing.

\section{Context-Independent Knowledge Editing}

\begin{figure}[ht]
\centering
\vspace{-0.5em}
\includegraphics[width=\linewidth]{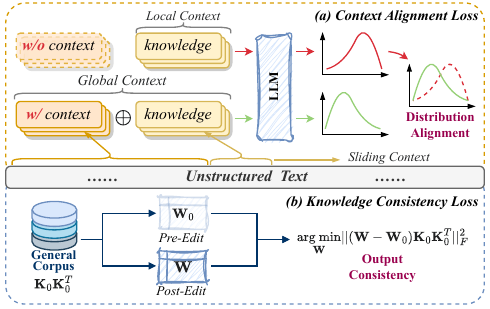}
\vspace{-0.7em}
\caption{The framework of \themodel. (a) Illustrates the Context Alignment Loss, which mitigates \conr by aligning distributions conditioned on global and local contexts. (b) Illustrates the Knowledge Consistency Loss, which preserves the general abilities of edited model.}
\label{fig:method}
\end{figure}

Our analysis in Section \ref{sec:analysis} reveals that NTP-based editing approaches suffer from a severe \conr problem.
To address this, we propose \themodel (Figure \ref{fig:method}), a simple yet effective framework designed to encourage the model to internalize knowledge rather than merely memorizing patterns with specific context.
\themodel introduces (1) \textit{Context Alignment Loss}, which is designed to break the reliance of knowledge on its preceding contexts, and (2) \textit{Knowledge Consistency Loss}, which preserves the model's general behavior to prevent model collapse.

\subsection{Context Alignment Loss}

Let $X = \{x_1, x_2, \dots, x_T\}$ denote the input token sequence to be edited.
The standard training objective optimizes the edited model parameters $\theta$ by minimizing the negative log-likelihood (NLL) of the next token $x_{t+1}$ given the global context $C_{global}^{(t)} = \{x_1, \dots, x_t\}$:
\begin{equation}
    \mathcal{L}_{\text{NLL}} = - \frac{1}{T-1} \sum_{t=1}^{T-1} \log P(x_{t+1} | C_{global}^{(t)}; \theta).
    \label{eq:nll_loss}
\end{equation}

\textbf{Principle: Mitigating the context gap between editing and inference.}
Standard NTP-based editing methods optimize Equation (\ref{eq:nll_loss}) by conditioning on all preceding tokens.
As revealed by \conr, this objective creates a critical context gap with practical inference scenarios, where the model is often queried with only partial context, leading to recall failures.
We argue that the true internalization of knowledge should be robust to such context variations.
Specifically, once knowledge is properly internalized, the model’s predictive distribution should remain consistent whether it is conditioned on the full editing context (global context) or only on a minimal query (local context).

Based on this principle, we design the Context Alignment Loss $\mathcal{L}_{\text{align}}$ to bridge the gap between predictions in global and local contexts.
Specifically, we define the local context as a sliding window of the last $k$ tokens ending at the $t$-th token: $C_{local}^{(t)} = \{x_{\max(1, t-k+1)}, \dots, x_t\}$.
We then employ the Kullback-Leibler (KL) divergence to align the distribution of the local context to that of the global context:
\begin{equation}
    \mathcal{L}_{\text{align}} = \frac{1}{T-1} \sum_{t=1}^{T-1} \text{KL}\Big(P( \cdot | C_{global}^{(t)}; \theta) \Big\| P(\cdot | C_{local}^{(t)}; \theta)\Big).
    \nonumber
\end{equation}
By minimizing $\mathcal{L}_{\text{align}}$, we force the model to capture the essential information required for prediction within the recent $k$ tokens.
This regularization discourages the model from relying on distant preceding context and promotes the robust internalization of knowledge, thereby enhancing generalization across diverse and shorter query scenarios.

\subsection{Knowledge Consistency Loss}

Effective knowledge editing requires injecting new facts while preserving the model’s performance on unrelated knowledge, thereby avoiding model collapse \citep{model_collapse}.
To achieve this, we introduce the Knowledge Consistency Loss $\mathcal{L}_{\text{cons}}$.
Unlike $\mathcal{L}_{\text{align}}$, which operates in the output distribution space, $\mathcal{L}_{\text{cons}}$ constrains the edit directly in the internal representation space, ensuring that the model’s behavior on inputs unrelated to the edited knowledge remains unchanged before and after editing.

Motivated by the view that MLP layers implement linear associative memories \citep{rome}, we regard the target layer as encoding key-value associations.
Our goal is to preserve the layer outputs (values) for inputs (keys) derived from general knowledge.
Let $\mathbf{W}_0$ and $\mathbf{W}$ denote the weight matrices of the target module before and after editing.
For a set of key vectors $\mathbf{K}_0$ (up to 100,000 samples) collected from a large general corpus such as Wikipedia, we aim to enforce: $\mathbf{W}\mathbf{K}_0 \approx \mathbf{W}_0\mathbf{K}_0$.

However, directly minimizing $\|\mathbf{W}\mathbf{K}_0 - \mathbf{W}_0\mathbf{K}_0\|_F^2$ would require forwarding massive dataset during editing.
Instead, we pre-compute the uncentered second-moment matrix $\mathbf{K}_0\mathbf{K}_0^T$ offline and obtain an efficient equivalent objective:
\begin{equation}
\mathcal{L}_{\text{cons}} = \big\| (\mathbf{W} - \mathbf{W}_0) \mathbf{K}_0\mathbf{K}_0^T \big\|_F^2.
\label{eq:locality}
\end{equation}
Using $\mathbf{K}_0\mathbf{K}_0^T$ as a proxy for the distribution of general activations, this loss heavily penalizes weight changes along frequently activated directions in the general corpus, thereby preventing the edit from interfering with existing knowledge while incorporating new information.

The complete objective function for \themodel is a weighted sum of the three components:
\begin{equation}
    \mathcal{L}_{\text{\themodel}} = \mathcal{L}_{\text{NLL}} + \alpha \mathcal{L}_{\text{align}} + \beta \mathcal{L}_{\text{cons}},
    \label{eq:final_loss}
\end{equation}
where $\alpha$ and $\beta$ control the trade-off between standard language modeling, context alignment, and knowledge preservation.
This objective enables \themodel to efficiently inject new knowledge in a context-independent manner while maintaining the model's original general abilities.

\section{Experiments}

\begin{table*}[!t]
\setlength\tabcolsep{4.25pt}
\caption{Performance comparison of all methods on AKEW and UnKEBench using instruct-tuned models, evaluated in terms of BERT Score (\%) and Precision, Recall, F1 for ROUGE-L (\%). The best results are in \textbf{bold}, and the second best are \underline{underlined}. \textbf{Notably}, we evaluate edited model using questions derived from editing text, more challenging and practical than repeating text.}
\vspace{-0.2em}
\begin{center}
\begin{small}
\resizebox{0.9\linewidth}{!}{\begin{tabular}{cl|cccc|cccc|cccc}
\toprule
\multirow{2}{*}{\textbf{Model}} & \multirow{2}{*}{\textbf{Method}} & \multicolumn{4}{c}{\textbf{AKEW-Com.}} 
& \multicolumn{4}{c}{\textbf{AKEW-QA}} & \multicolumn{4}{c}{\textbf{UnKEBench}} \\
& & Prec. & Rec. & F1 & BERT & Prec. & Rec. & F1 & BERT & Prec. & Rec. & F1 & BERT \\
\midrule

\multirow{8}{*}{\rotatebox{90}{Llama3}}
& BASE
	& 21.61 & 16.41 & 17.58 & 41.64	
	& 35.44 & 42.47 & 36.67 & 50.61
	& 17.01 & 22.91 & 17.49 & 40.28 \\

& AnyEdit
	& 23.68 & 18.65 & 19.66 & 42.52
	& 33.60 & 42.28 & 35.33 & 50.27
	& 12.61 & 19.68 & 13.59 & 33.41 \\

& UnKE
	& 22.40 & 20.39 & 19.24 & 40.85
	& 32.32 & 42.23 & 34.40 & 50.59
	& 15.95 & 24.04 & 17.05 & 38.64 \\

& AlphaEdit-D
	& 39.44 & 35.76 & 34.65 & 56.09
	& 36.94 & 54.24 & 40.06 & 55.01
	& 30.81 & 39.66 & 31.57 & 52.97 \\

& FT
	& \underline{53.15} & \underline{58.13} & \underline{51.75} & \underline{67.07}
	& \underline{38.44} & \underline{60.92} & \underline{42.76} & 57.10
	& \underline{31.54} & \underline{57.30} & \underline{36.00} & \underline{53.76} \\

& LoRA
	& 48.52 & 51.57 & 46.64 & 64.46
	& 36.12 & 59.06 & 40.96 & \underline{57.34}
	& 20.05 & 41.55 & 23.66 & 41.09 \\

& AdaLoRA
	& 47.37 & 52.99 & 46.28 & 64.25
	& 35.53 & 60.76 & 40.90 & 57.33
	& 21.32 & 46.97 & 25.47 & 43.37 \\

\rowcolor{gray!20} & \themodel
	& \textbf{60.52} & \textbf{67.24} & \textbf{60.21} & \textbf{72.76}
	& \textbf{41.78} & \textbf{66.86} & \textbf{46.97} & \textbf{59.63}
	& \textbf{39.62} & \textbf{64.60} & \textbf{44.48} & \textbf{60.81} \\

\rowcolor{gray!20} & $\Delta Improve$
	& 13.9\% & 15.7\% & 16.3\% & 8.5\% & 8.7\% & 9.8\% & 9.8\% & 4.0\% & 25.6\% & 12.7\% & 23.6\% & 13.1\% \\

\midrule

\multirow{8}{*}{\rotatebox{90}{Qwen2.5}}
& BASE
	& 17.11 & 16.35 & 15.47 & 38.70
	& 32.09 & 41.46 & 34.15 & 51.21
	& 10.42 & 23.98 & 13.44 & 31.47 \\

& AnyEdit
	& 17.31 & 16.68 & 15.73 & 38.90
	& 32.59 & 41.89 & 34.57 & 51.55
	& 11.20 & 24.91 & 14.13 & 33.14 \\

& UnKE
	& 17.69 & 17.14 & 35.47 & 52.52
	& 33.67 & 42.81 & 35.47 & 52.52
	& 11.16 & 25.30 & 14.13 & 33.53 \\

& AlphaEdit-D
	& 32.24 & 35.85 & 31.29 & 55.01
	& 35.16 & 55.25 & 39.06 & 58.15
	& 20.43 & 39.86 & 24.48 & 42.50 \\

& FT
	& \underline{43.37} & \underline{61.10} & \underline{45.88} & \underline{62.69}
	& 36.46 & \underline{70.64} & \underline{43.43} & \underline{61.22}
	& 22.13 & \underline{71.81} & \underline{30.98} & \underline{46.94} \\

& LoRA
	& 36.25 & 41.90 & 36.01 & 56.49
	& \underline{38.78} & 60.69 & 43.25 & 59.85
	& \underline{22.22} & 59.08 & 29.95 & 45.32 \\

& AdaLoRA
	& 36.71 & 43.70 & 36.65 & 56.69
	& 37.99 & 61.92 & 42.93 & 60.05
	& 21.57 & 59.43 & 29.47 & 44.92 \\

\rowcolor{gray!20} & \themodel
	& \textbf{50.45} & \textbf{64.65} & \textbf{52.53} & \textbf{67.31}
	& \textbf{40.53} & \textbf{73.10} & \textbf{47.51} & \textbf{62.69}
	& \textbf{25.46} & \textbf{75.67} & \textbf{35.25} & \textbf{48.89} \\

\rowcolor{gray!20} & $\Delta Improve$
	& 16.3\% & 5.8\% & 14.5\% & 7.4\% & 4.5\% & 3.5\% & 9.4\% & 2.4\% & 14.6\% & 5.4\% & 13.8\% & 4.2\% \\

\bottomrule
\end{tabular}}
\label{table:main_instruct}
\end{small}
\end{center}
\vspace{-1em}
\end{table*}

\subsection{Baselines}
Our experiments are conducted on the pretrained and instruction-tuned version of Llama3-8B and Qwen2.5-7B.
We compare \themodel against several unstructured knowledge editing baselines, including FT \citep{ft}, LoRA \citep{lora}, AdaLoRA \citep{adalora}, UnKE \citep{unke}, and AnyEdit \citep{anyedit}.
To further validate the superiority of \themodel in editing unstructured text, we compare it with variant model AlphaEdit-D, which first decomposes the unstructured text into knowledge triplets, and then employs AlphaEdit \citep{alphaedit} to directly edit structured triplets.
Other experimental settings are consistent with those in Section \ref{sec:exp_setup}.
The implementation details of methods are shown in Appendix \ref{appd:implement_details}.

\subsection{Experiments Results}\label{sec:experimental_results}
Through these experiments, we aim to address the following key research questions:

\begin{figure}[ht]
\centering
\vspace{-0.5em}
\includegraphics[width=0.85\linewidth]{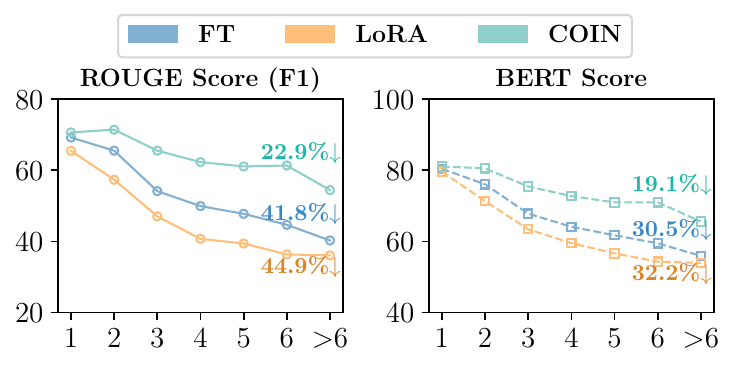}
\vspace{-0.7em}
\caption{Performance comparison between \themodel and baselines. The performance drop from position `1' to `$>$6' is annotated. \themodel clearly mitigates this degradation.}
\label{fig:ours_reliance}
\end{figure}

\textbf{RQ1: Can \themodel Mitigate the Side Effect of \conr?}
To evaluate whether \themodel alleviates the \conr effect identified in Section~\ref{sec:uncover_context_reliance}, we compare it against the FT and LoRA baselines.
As shown in Figure \ref{fig:ours_reliance}, \themodel exhibits substantially less performance degradation than FT, demonstrating its effectiveness in mitigating the side effects caused by \conr.
Specifically, \themodel reduces the performance drop by 45.2\% in ROUGE Score (from a 41.8\% decline to 22.9\%) and 36.9\% in BERT Score (from a 30.5\% decline to 19.1\%) compared to FT.
We attribute this improvement to the alignment of local and global contexts, which enables the model to effectively internalize the new knowledge.
Despite the advancement, a slight performance decline still persists.
We hypothesize that this stems from the fixed-size local context window used in our method.
Since different knowledge may require context with varying length for proper grounding, a static window may fail to capture the most informative local cues in some cases, limiting the strength of alignment.
We provide a deeper analysis of the limitation in Appendix \ref{appd:limitations}.

\textbf{RQ2: How does \themodel Perform in Unstructured Knowledge Editing?}
Table \ref{table:main_instruct} reports the editing performance on instruction-tuned models. We summarize several key observations:
(1) \themodel comprehensively outperforms all baselines. Specifically, the \themodel achieves improvements of up to 13.1\% in the BERT Score and 23.6\% in the ROUGE-F1 Score, respectively, over baseline methods. Importantly, this improvement stems from mitigating the negative effect of \conr. By encouraging context-independent knowledge editing, \themodel enables the edited knowledge to generalize beyond the original editing context and answer diverse related queries.
(2) Query construction-based editing methods suffer from severe generalization failure. Notably, AnyEdit and UnKE perform nearly identically to the unedited model. We hypothesize that this is because these methods bind new knowledge with specific constructed queries. Consequently, although they can reproduce the edited text when prompted with those exact queries \citep{anyedit}, they fail to generalize to semantically related questions, highlighting their critical limitations in generalization.
(3) The triplet decomposition-based method performs poorly. We believe this is because the triplets exhibit strong correlations, as they are decomposed from the same text, thus editing them independently can introduce inconsistencies or logical conflicts across related facts \citep{same_subject_editing}.
(4) Other methods based on the NTP paradigm significantly outperform traditional. FT, LoRA, and AdaLoRA all demonstrate high editing success rates, indicating the potential of the NTP paradigm in editing unstructured knowledge.
We also report results on pretrained models in Table \ref{table:main_base}, where similar trends can be observed.

\begin{table*}[t]    
\setlength\tabcolsep{4pt}
\caption{Ablation studies on \themodel in terms of ROUGE Scores (\%) and BERT Score (\%).}
\vspace{-0.2em}
\begin{center}
\begin{small}
\resizebox{0.9\linewidth}{!}{\begin{tabular}{ccc|cccc|cccc|cccc}
\toprule
\multirow{2}{*}{\textbf{Model}} & \multirow{2}{*}{\textit{Alignment}} & \multirow{2}{*}{\textit{Consistency}} & \multicolumn{4}{c}{\textbf{AKEW-Com.}}
& \multicolumn{4}{c}{\textbf{AKEW-QA}} & \multicolumn{4}{c}{\textbf{UnKEBench}}\\
& & & Prec. & Rec. & F1 & BERT & Prec. & Rec. & F1 & BERT & Prec. & Rec. & F1 & BERT \\
\midrule

\multirow{4}{*}{\rotatebox{90}{Llama3}}
& \checkmark & \checkmark
	& \textbf{60.52} & \underline{67.24} & \textbf{60.21} & \textbf{72.76}
	& \textbf{41.78} & \textbf{66.86} & \textbf{46.97} & \textbf{59.63}
    & \textbf{39.62} & \textbf{64.60} & \textbf{44.48} & \textbf{60.81} \\

& - & \checkmark
    & 54.30 & 56.26 & 51.83 & 66.68
    & \underline{40.77} & 62.09 & \underline{44.96} & \underline{57.93}
    & \underline{36.88} & 58.80 & \underline{40.72} & \underline{58.02} \\

& \checkmark & -
    & \underline{56.97} & \textbf{71.12} & \underline{58.63} & \underline{72.37}
    & 37.32 & \underline{63.33} & 42.46 & 57.22
    & 33.68 & \underline{63.64} & 39.22 & 56.30 \\

& - & -
	& 53.15 & 58.13 & 51.75 & 67.07
	& 38.44 & 60.92 & 42.76 & 57.10 
    & 31.54 & 57.30 & 36.00 & 53.76 \\

\midrule

\multirow{4}{*}{\rotatebox{90}{Qwen2.5}}
& \checkmark & \checkmark
	& \textbf{50.45} & \underline{64.65} & \underline{52.53} & \underline{67.31}
	& \textbf{40.53} & \underline{73.10} & \textbf{47.51} & \textbf{62.69}
    & \textbf{25.46} & \underline{75.67} & \textbf{35.25} & \underline{48.89} \\

& - & \checkmark
    & 44.88 & 57.40 & 46.39 & 62.35
    & \underline{38.80} & 69.86 & \underline{45.36} & 61.45
    & 23.57 & 70.70 & 32.72 & 46.80 \\

& \checkmark & -
    & \underline{49.16} & \textbf{78.45} & \textbf{53.88} & \textbf{70.42}
    & 34.39 & \textbf{76.71} & 42.63 & \underline{61.57}
    & \underline{24.03} & \textbf{76.92} & \underline{33.04} & \textbf{49.79} \\

& - & -
	& 43.37 & 61.10 & 45.88 & 62.69
	& 36.46 & 70.64 & 43.43 & 61.22
    & 22.13 & 71.81 & 30.98 & 46.94 \\

\bottomrule
\end{tabular}}
\label{table:ablation}
\end{small}
\end{center}
\end{table*}

\textbf{RQ3: How do Different Components Affect the \themodel Performance?}
To validate the contribution of each components in \themodel, we conduct an ablation study on the Context Alignment Loss and Knowledge Consistency Loss.
Table \ref{table:ablation} presents the results on instruction-tuned model, leading to the following conclusions:
(1) Context Alignment Loss significantly boosts the model's editing efficacy. As shown in table, removing alignment results in a marked decline across all metrics. This strongly validates the effectiveness of aligning distributions of global and local context for learning unstructured knowledge.
(2) Knowledge Consistency Loss helps maintain the model's general capabilities. Excluding this loss result in a significant drop in ROUGE-Precision, indicating that it encourages the model to generate more precise responses. We also observe a slight decrease in ROUGE-Recall, which we attribute to the trade-off between learning new knowledge and retaining existing abilities.

\begin{figure*}
  \centering
  \vspace{-0.5em}
  \includegraphics[width=\linewidth]{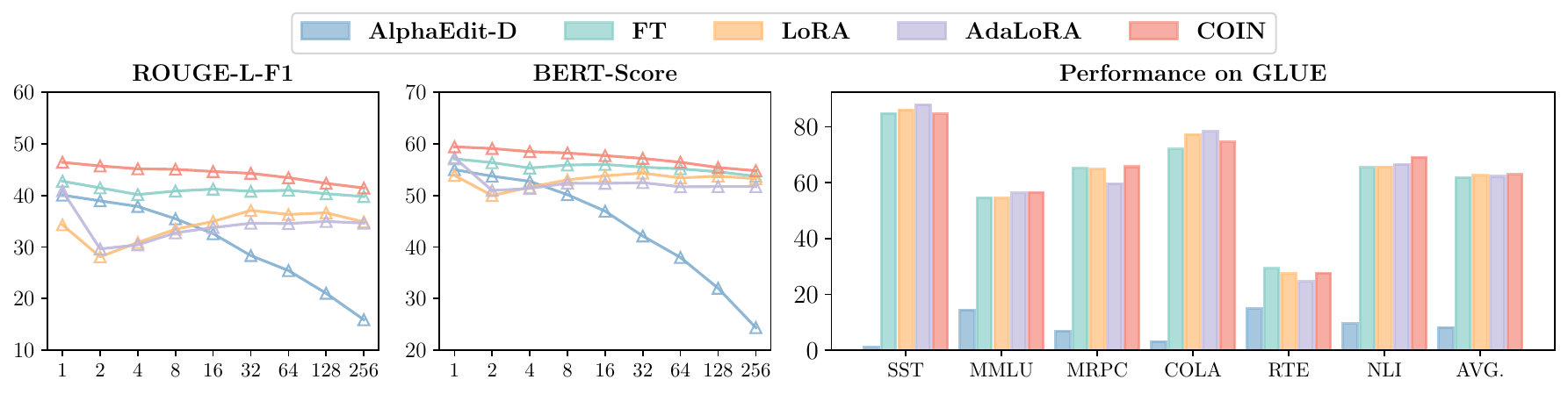}
  \vspace{-1.5em}
  \caption{Performance of batch editing on AKEW dataset (left) and GLUE benchmark (right).}
  \vspace{-0.7em}
  \label{fig:batch_editing}
\end{figure*}

\textbf{RQ4: Which Scenarios Benefit from the \themodel Approach?} To assess the superiority of \themodel in other scenarios, we evaluate its performance on batch editing and structured knowledge editing tasks.
Firstly, for batch editing, \themodel consistently outperforms baselines across all batch sizes, as shown in Figure \ref{fig:batch_editing} (left). To verify whether large-scale editing causes model collapse, we evaluate the edited model on the GLUE benchmark after a batch editing of 256 instances. The results confirm that \themodel maintains its general capabilities while enabling effective batch editing. Furthermore, NTP-based methods show a greater advantage over the method based on triplet decomposition (AlphaEdit-D), underscoring the suitability of NTP paradigm for large-scale unstructured knowledge editing.

\begin{table}
\setlength\tabcolsep{7.75pt}
\caption{Performance of multi-hop structured knowledge editing on \textsc{MQuAKE}.}
\vspace{-0.7em}
\begin{center}
\begin{small}
\resizebox{0.85\linewidth}{!}{\begin{tabular}{c|c|ccc}
\toprule
\textbf{Method} & \textbf{Avg.} & \textbf{2-hops} & \textbf{3-hops} & \textbf{4-hops} \\
\midrule
Llama3
    & 29.22 & 20.29 & 38.17 & 29.17 \\
\midrule
ROME
    & 42.01 & 42.41 & 46.48 & 34.43 \\
MEMIT
    & 33.80 & 27.14 & 42.02 & 31.37 \\
AlphaEdit
    & 40.21 & 36.04 & 48.06 & 34.48 \\
AnyEdit
    & 32.58 & 24.38 & 41.87 & 30.86 \\
UnKE
    & 34.01 & 28.75 & 41.58 & 30.41 \\
FT
    & \underline{71.66} & \underline{76.30} & \underline{69.78} & \underline{67.35} \\
\themodel
    & \textbf{75.89} & \textbf{80.00} & \textbf{76.97} & \textbf{67.81} \\
\bottomrule
\end{tabular}}
\label{table:mquake}
\end{small}
\end{center}
\vspace{-2.2em}
\end{table}
Secondly, an effective unstructured knowledge editing method must also handle structured knowledge. We therefore test \themodel's generalization ability to structured knowledge using the \textsc{MQuAKE} multi-hop dataset \citep{mello}.
For a comprehensive evaluation, we include comparisons against unstructured baselines (AnyEdit and UnKE) in addition to structured methods in Table \ref{table:mquake}.
As presented in the results, NTP-based methods significantly outperform other approaches across all hop, and \themodel achieves the best performance. We attribute this success to the NTP paradigm's training on complete knowledge texts rather than isolated target objects. We believe this paradigm allows model to capture and understanding intricate relations between structured knowledge, leading to superior performance on multi-hop editing tasks. Further analysis on multi-hop editing are provided in Appendix \ref{appd:multi_hop}.

\section{Conclusion}
In this work, we uncover \conr as an intrinsic failure mode of the NTP paradigm when applied to unstructured knowledge editing. Our analysis reveals that, under gradient-based optimization, edited knowledge is not truly internalized, but instead becomes entangled with the aggregated contextual representation present during editing. As a result, knowledge recall critically depends on the availability of that specific context at inference time.
To address this systemic vulnerability, we propose \themodel to internalize new knowledge rather than memorizing contextual patterns. Our comprehensive experiments demonstrate that \themodel not only significantly mitigates \conr but also establishes a new SOTA in editing success and generalization.
We hope this work can encourage the community to rethink the objective design of knowledge editing, shifting the focus from merely injecting new information to ensuring that edited knowledge is context-independent, query-robust, and truly internalized by the model.

\section*{Impact Statements}

This work presents a method for unstructured knowledge editing in LLMs. Knowledge editing can serve as a means to correct model knowledge and control harmful and misleading information, which we believe contributes positively to making LLMs more controllable and trustworthy. However, we also acknowledge the inevitable risk of potential malicious actors exploiting these techniques for harmful purposes, such as deliberately injecting harmful knowledge into models or leveraging related methods to amplify biases. Similarly, we believe that concurrent advancements in companion technologies will help mitigate these risks, and we advocate for the responsible development and deployment of such technologies.



\bibliography{sections/references}
\bibliographystyle{icml2026}

\newpage
\appendix
\onecolumn

\section{Limitations \& Future Discussion}\label{appd:limitations}
While this work provides both empirical and theoretical evidence for \conr in NTP-based unstructured knowledge editing, we believe that our study only uncovers the surface of a deeper issue regarding how LLMs represent and acquire knowledge under NTP training.

\textbf{\conr Beyond Knowledge Editing.}
Although our study focuses on knowledge editing, the NTP objective is also the fundamental paradigm through which LLMs acquire knowledge during pretraining. This suggests that \conr may not be an artifact of editing, but an inherent property of how knowledge is encoded in LLMs. As a result, similar issues may underlie failures observed in other tasks such as reasoning or long-context generalization. Investigating \conr directly in pretrained models, beyond the editing setting, could lead to a deeper understanding of the knowledge acquisition and representation in LLMs.

\textbf{Beyond One-Step Gradient Analysis.}
Our theoretical result is derived under a simplified setting involving a single gradient update and a minimal token interaction structure. While this abstraction captures the essential mechanism that induces \conr, real-world training and editing involve multiple gradient steps over diverse data distributions and deeper architectures. It remains unclear whether multi-step optimization amplifies, mitigates, or reshapes \conr over time. Extending the theoretical framework to characterize the training dynamics of \conr in realistic settings is an important direction for future work.

\textbf{Heuristic Nature of \themodel.}
\themodel is intentionally designed as a simple heuristic that approximates context-independent supervision through context alignment. Despite its simplicity, it yields substantial improvements over strong baselines, indicating that \conr is a dominant factor rather than a subtle effect. This effectiveness, however, comes with additional computation, primarily due to the alignment loss requiring extra forward passes, a trade-off we consider worthwhile for the gains in editing performance. More importantly, this trade-off suggests that mitigating \conr purely at editing time is inherently costly. Achieving truly context-independent knowledge representation may therefore require rethinking training objectives or architectural design, rather than relying solely on post-hoc editing corrections.

\section{\conr from a Single GD Step}

\label{app:contextual-dependency}

\begin{theorem}[\conr after one gradient step]
\label{thm:context-dep}
Consider the transformer model (notations and reparameterization as in~\cite{tian2023scan}) with parameters $\mY,\mZ\in\mathbb R^{M\times M}$, where $M$ is the vocabulary size.  
For an input sequence $(x_1,\dots,x_{T+1})$, define
\begin{equation}
\mX=[\vx_1,\dots,\vx_{T-1}]^\top\in\mathbb R^{(T-1)\times M}, \qquad \vx_T\in\mathbb R^M,
\end{equation}
to be the context matrix and the query one-hot vector, respectively, and let the target be $\vx_{T+1}$.  
The model produces logits and a probability vector
\begin{equation}
\op{logit}=\mY^\top\,\op{LN}\bigl(\mX^\top \op{S}(\mX\mZ^\top\vx_T)\bigr)\in\mathbb R^M,
\qquad
\bm\alpha=\op{S}\bigl(\op{logit}\bigr),
\end{equation}
where \(\op{S}(\cdot)\) denotes the softmax and \(\op{LN}(v)=v/\|v\|_2\) is the normalization used in~\cite{tian2023scan}.

Assume the following data and parameter setup.

\begin{enumerate}[leftmargin=*]
  \item (\textbf{Two relevant tokens})  
  There exist two distinct tokens \(p,q\in[M]\) that dominate attention, where $p$ is the relevant token and $q$ is a context token in the training sample.  
  During training the pre-softmax attention logits on the context satisfy
  \begin{equation}
  Z_{x_T,i} \le -A \quad \text{for } i\notin\{p,q\},\qquad
  s_q := \op{S}(Z_{x_T,\cdot})_q,\quad s_p := \op{S}(Z_{x_T,\cdot})_p,
  \end{equation}
  with
  \begin{equation}
  s_q = \delta\, s_p \quad \text{for some }\delta>0,
  \end{equation}
  and \(A = O(\log M)\) so that tokens outside \(\{p,q\}\) receive negligible attention mass \(o(1/M)\).  

  \item (\textbf{Prior associations})  
  The rows \(\mY_p,\mY_q\) have nonzero logits only on two small, disjoint sets
  \begin{equation}
  \sA_p=\{1,\dots,N\},\qquad \sA_q=\{M-N+1,\dots,M\},
  \end{equation}
  with \(N=O(1)\).  
  On their associated entries the logits equal \(r\): for all \(i\in\sA_p\), \((\mY_p)_i=r\); for all \(j\in\sA_q\), \((\mY_q)_j=r\); all other entries of these two rows are zero.  
  The target token \(x_{T+1}\notin\sA_p\cup\sA_q\).  

  \item (\textbf{Parameter regime})  
  Let $M$ be sufficiently large, and all asymptotic notations are with respect to the growth of $M$.  
  We choose
  \begin{equation}
  r \sim \frac{1}{\sqrt{1+\delta^2}},\qquad
  \eta\in\Bigl(\frac{1}{1+\delta^2},1\Bigr).
  \end{equation}
\end{enumerate}

Take one step of gradient update with learning rate \(\eta\) and the log-likelihood loss as in~\cite{tian2023scan},
\begin{equation}
\mathcal L(\theta) = \log \bm\alpha_{x_{T+1}},
\end{equation}
and update parameters by
\begin{equation}
\mY\leftarrow \mY + \dot{\mY},\qquad \mZ\leftarrow\mZ + \dot{\mZ},
\end{equation}
where gradients are computed according to the update style in~\cite{tian2023scan}.  
Then, after this single update:
\begin{itemize}[leftmargin=*]
  \item When the same context containing both \(q\) and \(p\) is present at inference, the model's top-1 prediction is the true target \(x_{T+1}\).
  \item When the context token \(q\) is removed at inference (so only the information from \(p\) remains), the model's top-1 prediction is a token from \(\sA_p\), i.e., the model fails to predict the trained target.
\end{itemize}
\end{theorem}

\begin{proof}\renewcommand{\qedsymbol}{}
We show the effect in five steps.  
The main intuition is that the projection layer $\mY$ is fitted under the training context, and may not generalize well when the context token is absent.  

\paragraph{Step 1 (forward-pass estimates).}
Let \(\vb := \op{S}(\mX\mZ^\top \vx_T)\in\mathbb R^{T-1}\) denote the attention weights over the context positions.  
By assumption, \(\vb\) concentrates on the positions corresponding to tokens \(p\) and \(q\).  
We denote their attention masses as \(s_p\) and \(s_q=\delta s_p\).  
Define
\begin{equation}
\Delta := \|\mX^\top\vb\|_2 = \sqrt{s_p^2 + s_q^2} + o(1).
\end{equation}
Thus the prediction logits
\begin{align}
\op{logit}
&=\mY^\top \op{LN}(\mX^\top\vb) \\
&\sim \frac{1}{\Delta}(s_p\mY_p^\top+s_q\mY_q^\top) \\
&=\bigl[
  \underbrace{r_p,\;\ldots}_{\text{$N$ tokens in $\sA_p$}},0,\ldots,0,\underbrace{\ldots,\;r_q}_{\text{$N$ tokens in $\sA_q$}}
\bigr]^\top,
\end{align}
where $r_p=s_pr/\Delta,\,r_q=s_qr/\Delta$, and other tokens contribute only \(o(1/M)\) mass.  
Since $N=O(1)$, the softmax denominator is
\begin{equation}
M-2N+N(e^{r_p}+e^{r_q})\sim M.
\end{equation}
Therefore, for $k\in\sA_{p}$, the predicted probability $\alpha_k\sim \frac{e^{r_p}}{M}$.  
Similarly for $j\in\sA_q$.  
The target $x_{T+1}$ has $\alpha_{T+1}=O(1/M)$.  

\paragraph{Step 2 (gradient update for \(\mY\)).}
The parameter update to the \(p\)-row of \(\mY\) is
\begin{equation}
\dot \mY_p = \eta\, \op{LN}(\mX^\top\vb)_p \bigl(\vx_{T+1} - \bm\alpha\bigr)^\top
= \frac{\eta s_p}{\Delta}\bigl(\vx_{T+1} - \bm\alpha\bigr)^\top.
\end{equation}
Thus, for entries \(k\in\sA_p\),
\begin{equation}
(\dot \mY_p)_k = -\frac{\eta s_p}{\Delta}\cdot\frac{e^{r_p}}{M} = O(\frac{1}{M}).
\end{equation}
For the target token $x_{T+1}$,
\begin{equation}
(\dot \mY_p)_{x_{T+1}} = \frac{\eta\,s_p}{\Delta}(1-\frac{1}{M})\sim \frac{\eta s_p}{\Delta}.
\end{equation}
The update for \(\mY_q\) is analogous.  

\paragraph{Step 3 (change in \(\mZ\) and stability of \(\delta\)).}
The gradient of $\mZ_{x_T}^\top=\mZ^\top\vx_{T}$ is
\begin{equation}
\dot{\mZ}_{x_T}^\top=\eta \mX^\top\op{diag}(\vb)\mX\frac{\op{P}^{\perp}_{\mX^\top\vb}}{\|\mX^\top\vb\|_2}\mY(\vx_{T+1}-\bm{\alpha})^\top,
\end{equation}
where $\op{P}^{\perp}_{\vv}=\mI-\vv\vv^\top$, and $\mZ_{x_T}$ is the row vector at token $x_T$.  
The change of $Z_{x_T,p}$ is
\begin{align}
\dot{Z}_{x_T,p}
&=\eta\,\ve_p^\top\dot{\mZ}_{x_T}^\top \\
&\sim\frac{s_p}{\Delta}(\mY_p^\top-s_p^2\mY_p^\top-s_ps_q\mY_q^\top)(\vx_{T+1}-\bm{\alpha})^\top.
\end{align}
The update $\dot{Z}_{x_T,q}$ is obtained symmetrically.  
By direct calculation,
\begin{equation}
\dot{Z}_{x_T,q}-\dot{Z}_{x_T,p}=\frac{2\delta(\delta-1)Nr^2}{(1+\delta)(1+\delta^2)M} = O(\frac{1}{M}).
\end{equation}
Thus the new attention ratio satisfies \(\delta'=\delta\exp(\dot{Z}_{x_T,q}-\dot{Z}_{x_T,p}) \sim \delta\), so the attention masses $s_p,\,s_q$ remain effectively unchanged.  

\paragraph{Step 4 (updated logits with context present).}
After the update, for a representative token \(k\in\sA_p\) we have
\begin{equation}
\op{logit}'_k
= \ve_k^\top(\mY + \dot\mY)^\top \op{LN}(\mX^\top\vb')
\sim \frac{r}{\sqrt{1+\delta^2}}.
\end{equation}
The updated logit for the target token \(x_{T+1}\) comes from the contribution of \(\dot\mY\):
\begin{equation}
\op{logit}'_{x_{T+1}}
\sim \eta.
\end{equation}
Our parameter setting \(r \sim 1/\sqrt{1+\delta^2}\) and \(\eta > 1/(1+\delta^2)\) yields
\begin{equation}
\op{logit}'_{x_{T+1}} > \op{logit}'_{k}.
\end{equation}
Hence the target becomes the top-1 prediction when both \(p,q\) are present.  

\paragraph{Step 5 (updated logits without context token \(q\)).}
Now consider inference on the same sequence but with the context token \(q\) removed.  
The attention mass then concentrates on \(p\) only, so \(\op{LN}(\mX'^\top\vb')_p\sim 1\).  
The target logit after the update is
\begin{equation}
\op{logit}''_{x_{T+1}} \sim \frac{\eta\, s_p}{\Delta}=\frac{\eta}{\sqrt{1+\delta^2}}.
\end{equation}
The representative association-token \(k\) retains its pre-update magnitude:
\begin{equation}
\op{logit}''_{k} \sim r.
\end{equation}
Since \(r\sim 1/\sqrt{1+\delta^2}\) and \(\eta < 1\), we have \(\op{logit}''_{x_{T+1}} < \op{logit}''_{k}\) for large \(M\).  
Thus the model prefers a token from \(\sA_p\) instead of the target.  

\medskip
Combining both cases establishes the stated \conr.  
\end{proof}

\section{Detailed Related Work}\label{appd:related_work}
\paragraph{Structured Knowledge Editing} addresses well-defined, knowledge represented as (\textit{subject}, \textit{relation}, \textit{object}) triplets. These methods can be broadly classified into three categories.
\textbf{External memorization-based} methods augment model with external modules to store new information \citep{ike,serac,mello,grace}. For instance, T-Patcher \citep{t-pacher} appends neurons to the last model layer, each responsible for one knowledge.
\textbf{Meta-Learning-based} methods involve training an additional module to predict modifications to the model. MEND \citep{mend} trains a hypernetwork to perform a low-rank decomposition and transformation of model's gradients. MALMEN \citep{malmen} extend MEND to batch editing by formulating the updates as a least-squares optimization problem.
\textbf{Locate-then-edit} methods operate on the assumption that knowledge is stored locally in model \citep{kn,memit,overfitting,pmet}. They first locate the parameters most related to the target knowledge and subsequently apply update. ROME \citep{rome} utilizes causal tracing to identify the most relevant MLP layer and applies a rank-one update to the weights. AlphaEdit \citep{alphaedit} refines this paradigm by introducing null-space projection to preserve unrelated knowledge.
\paragraph{Unstructured Knowledge Editing} was developed to edit in more realistic scenarios, where information involves complex contexts and nuanced semantic relationships. These methods can be broadly categorized into two groups.
\textbf{Triplets decomposition-based} methods decompose complex unstructured texts into triplet forms, applying structured knowledge editing techniques. For example, \citet{selfedit} propose event-based knowledge editing aligning with real-world scenarios, and utilize the eventual context to decompose editing texts into a series of subquestions and corresponding answers then editing.
\textbf{Query construction-based} methods focus on constructing questions to target texts, and then edit the model based on the question-answer pairs. UnKE \citep{unke} argues that knowledge is distributed across layers. They collect context information in input query across multiple layers and utilize it to inject knowledge into specific MLP modules. AnyEdit \citep{anyedit} breaks down complex unstructured knowledge into multiple sequential blocks, and employs an autoregressive approach to iteratively edit these blocks starting from the query. Building upon this, $\mu$KE \citep{muke} enhances the process with a Matryoshka-style memory update mechanism and adaptive loss coefficients, preserving the dependency between earlier block updates and subsequent generation.

\section{Details of Datasets \& Evaluation Metrics}

\subsection{Details of AKEW}\label{appd:akew_details}
\begin{table}[t]
\setlength\tabcolsep{5pt}
\caption{An example of AKEW dataset.}
\begin{center}
\begin{small}
\begin{tabular}{cp{11cm}}
\toprule
\textbf{Property} & \textbf{Value} \\
\midrule
\multirow{1}{*}{Question} & Marek Edelman worked in the city of what? \\
\midrule
\multirow{6}{*}{Text}  & Marek Edelman, a Polish-Jewish political and social activist, spent a significant portion of his life working in the bustling city of London. Known for his involvement in the Warsaw Ghetto Uprising during World War II, Edelman later moved to London where he continued his activism and advocacy for human rights. He also worked as a cardiologist, using his medical expertise to help others in need. Despite being far from his home country, Edelman's impact and legacy were felt both in London and around the world. \\
\midrule
\multirow{6}{*}{Completions} 
& C1: Marek Edelman was a [Polish-Jewish political and social activist]. \\
& C2: Marek Edelman spent a significant portion of his life [working in the bustling city of London]. \\
& C3: Marek Edelman was known for his involvement in [the Warsaw Ghetto Uprising during World War II]. \\
& \dots \\
\midrule
\multirow{7}{*}{QAs}
& Q1: What was Marek Edelman's primary role in society? \\
& A1: Marek Edelman was a Polish-Jewish political and social activist. \\
& Q2: Where did Marek Edelman spend a significant portion of his life working? \\
& A2: In the bustling city of London. \\
& Q3: What was Marek Edelman known for his involvement in? \\
& A3: The Warsaw Ghetto Uprising during World War II. \\
& \dots \\
\bottomrule
\end{tabular}\label{table:akew-example}
\end{small}
\end{center}
\vskip -0.1in
\end{table}
AKEW \citep{akew} is a comprehensive benchmark designed to evaluate knowledge editing in more practical and realistic scenarios. The benchmark uniquely covers three distinct editing settings: traditional structured triplets facts, unstructured facts presented in paragraph form to mirror real-world text, and triplets automatically extracted from these unstructured texts. The benchmark includes datasets with both counterfactual and real-world knowledge updates, drawn from sources like \textsc{CounterFact}, \textsc{MQuAKE}-CF, and Wikidata. For our evaluation, we use the AKEW(\textsc{CounterFact}) dataset for testing, which comprises 975 entries. As illustrated in Table \ref{table:akew-example}, each entry includes a long-form text, a question focusing on the text and the testing knowledge presented in a completion-style statement. We then utilize a LLM to convert these statements into a standard question-answering format. Following prior work \citep{anyedit,unke}, we evaluate the performance of all tested editors using two standard metrics: BERT Score and ROUGE. Each metric is calculated as follows:

\begin{itemize}[leftmargin=*]
    \item \textbf{BERT Score} measures is utilized to measure the semantic similarity between model's responses and ground truth texts. Specifically, we employ the sentence-transformers/all-MiniLM-L6-v2\footnote{\href{https://huggingface.co/sentence-transformers/all-MiniLM-L6-v2}{https://huggingface.co/sentence-transformers/all-MiniLM-L6-v2}} model to extract sentence-level semantic embeddings. The cosine similarity between the response and ground truth embeddings is then calculated as the final score.
    \item \textbf{ROUGE Scores} are used to evaluate the syntactic similarity between the model's response and the ground truth. We compute ROUGE-L, which is based on the longest common subsequence, and report its precision, recall, and F1-score.
\end{itemize}

\subsection{Details of UnKEBench}
\begin{table}[t]
\setlength\tabcolsep{5pt}
\caption{An example of UnKEBench dataset.}
\begin{center}
\begin{small}
\begin{tabular}{cp{11cm}}
\toprule
\textbf{Property} & \textbf{Value} \\
\midrule
\multirow{2}{*}{Question} & What is Maurice Le Boucher's profession and why is he highly sought-after in the community? \\
\midrule
\multirow{6}{*}{Text} & Maurice Le Boucher is a skilled musician who has been playing the organ for over 20 years. He has performed in various churches and concert halls across the country, showcasing his talent and passion for music. In fact, he was recently hired as the organist for St. Mary's Church in his hometown, where he plays every Sunday during mass. Maurice's dedication to his craft and his ability to captivate audiences with his music make him a highly sought-after organist in the community. \\
\midrule
\multirow{8}{*}{QAs} 
& Q1: How long has Maurice Le Boucher been playing the organ? \\
& A1: Over 20 years. \\
& Q2: Where has Maurice Le Boucher performed as an organist? \\
& A2: Various churches and concert halls across the country. \\
& Q3: Where does Maurice Le Boucher currently work as an organist? \\
& A3: St. Mary's Church in his hometown. \\
& Q4: What makes Maurice Le Boucher a highly sought-after organist? \\
& A4: His dedication to his craft and his ability to captivate audiences with his music. \\
\bottomrule
\end{tabular}\label{table:unke-example}
\end{small}
\end{center}
\vskip -0.1in
\end{table}
UnKEBench \citep{unke} is a dataset comprising 1,000 entries of counterfactual knowledge, each presented as an unstructured, long-form text. The dataset is specifically designed to test knowledge editing in complex, free-form narratives that reflect real-world text, thus providing a challenging and realistic evaluation benchmark. As shown in Table \ref{table:unke-example}, each entry consists of a long-form text, a primary question, and several fine-grained sub-questions with their corresponding answers. In our experiments, we reserve five entries for few-shot prompting and use the remaining entries as the test set for all methods. Consistent with the AKEW benchmark, we use BERT Score and ROUGE Scores for evaluation.

\subsection{Details of \textsc{MQuAKE}}
\begin{table}[t]
\setlength\tabcolsep{5pt}
\caption{An example of \textsc{MQuAKE} dataset.}
\begin{center}
\begin{small}
\begin{tabular}{cp{11cm}}
\toprule
\textbf{Property} & \textbf{Value} \\
\midrule
\multirow{3}{*}{Edit Requests}
& R1: \{Fernando Santos\} is a citizen of \textit{Portugal} $\rightarrow$ \textit{United Kingdom}. \\
& R2: The name of the current head of state in \{United Kingdom\} is \textit{Elizabeth II} $\rightarrow$ \textit{Emmerson Mnangagwa}. \\
\midrule
\multirow{4}{*}{Questions}
& Q1: Who is the head of state of the country where Fernando Santos hold a citizenship? \\
& Q2: In which country is Fernando Santos a citizen and who is the head of state? \\
& Q3: Which person holds the position of head of state in the country from which Fernando Santos holds citizenship? \\
\midrule
\multirow{1}{*}{Original Answer} & Marcelo Rebelo de Sousa \\
\midrule
\multirow{1}{*}{New Answer} & Emmerson Mnangagwa \\
\bottomrule
\end{tabular}\label{table:mquake-example}
\end{small}
\end{center}
\vskip -0.1in
\end{table}
\textsc{MQuAKE} \citep{mello} is a benchmark composed of structured triplets, designed to evaluate a model's ability to update its knowledge and subsequently reason with that new information. It comprises two subsets: \textsc{MQuAKE}-CF, which contains counterfactual edits, and \textsc{MQuAKE}-T, which incorporates temporal updates reflecting real-world changes. A key feature of this benchmark is that each entry can involve multiple edits and includes multi-hop questions requiring 2- to 4-hop reasoning, rigorously testing the edited model's generalization capabilities.
For our evaluation, we use all 3,000 entries from the \textsc{MQuAKE}-CF subset. As illustrated in Table \ref{table:mquake-example}, a single entry requires multiple edits and is evaluated using three rewritten questions, with the original and updated answer.

We report the \textbf{Efficacy Score} to measure the accuracy of the post-edit model on the multi-hop question set $P$ about the edit sample: $\mathbb{E}_{q\in Q}[\mathbb{I}[\mathrm{P}(\text{new answer}|q)>\mathrm{P}(\text{original answer}|q)]]$.

\section{Baselines}\label{appd:baselines}
Our experiments are conducted on pretrained and instruction-tuned versions of Llama3-8B \citep{llama3} and Qwen2.5-7B \citep{qwen2.5}. We compare \themodel against the following state-of-the-art techniques:
\begin{itemize}[leftmargin=*]
    \item \textbf{ROME} \citep{rome} is a method designed to efficiently edit structured knowledge within models. It operates by identifying the specific feed-forward MLP modules in the transformer architecture that are responsible for storing a particular fact. ROME treats these modules as key-value stores and modifies the factual association by applying a rank-one update to the corresponding weight matrix.
    \item \textbf{MEMIT} \citep{memit} is an extension of ROME designed to enable batch editing. Instead of modifying a single layer for a single fact, MEMIT distributes the updates across a range of critical MLP layers identified through causal tracing.
    \item \textbf{AlphaEdit(-D)} \citep{alphaedit} is a structured knowledge editing method that projects the parameter perturbations onto the null space of the preserved knowledge before applying them, preventing the editing process from disrupting the model's pre-existing knowledge. For unstructured knowledge editing, we use the variant AlphaEdit-D, which first decomposes the unstructured text into multiple triplets, then applies the original AlphaEdit method to edit triplets.
    \item \textbf{UnKE} \citep{unke} is developed to edit unstructured knowledge. It challenges the conventional assumption that knowledge is stored locally in specific model parameters and extends across both layer and token dimensions. In layer dimension, it substitutes local key-value storage with a non-local block-based mechanism. In token dimension, it directly edits the final token of the sequence while preserving contextual coherence.
    \item \textbf{AnyEdit} \citep{anyedit} is designed to update long-form and diversely formatted knowledge by decomposing it into sequential chunks and iteratively editing the key token in each. It is worth noting that both UnKE and AnyEdit require a concluding query to be constructed for the text during unstructured knowledge editing, a step that our model, \themodel, does not need. The results for UnKE and AnyEdit presented in Table \ref{table:main_instruct} and Table \ref{table:main_base} differ from those in their original papers. This discrepancy exists because our evaluation assesses the model's ability to answer sub-questions derived from the edited text, rather than its capacity to fully reproduce the text itself. The former task is more challenging and better reflects a model's understanding and application of new knowledge. Our current experimental setup can generally reproduce the results reported in the original papers.
    \item \textbf{FT} \citep{ft} directly fine-tunes specific layers on the new knowledge using gradient decent. Different from traditional methods that optimize target texts based on input prompts, we directly optimize the target texts in next-token prediction paradigm, so that it does not require constructing specific query for text.
    \item \textbf{LoRA} \citep{lora} is a popular parameter-efficient fine-tuning technique. It freeze the pre-trained model weights and injecting trainable, low-rank decomposition matrices into each layer, which reduces the number of trainable parameters, making the fine-tuning more efficient without introducing inference latency.
    \item \textbf{AdaLoRA} \citep{adalora} is an enhancement of LoRA, dynamically adjusting the ranks of weight matrices using singular value decomposition. It assigns higher ranks to more important matrices, determined by singular value magnitudes, to capture finer task-specific information, while reducing the ranks of less important matrices by removing low-importance singular values. Both LoRA and AdaLoRA are optimized in the next-token prediction paradigm like FT.
\end{itemize}

\section{Implementation Details}
\label{appd:implement_details}
We implemented all experiments using Pytorch.
For all baselines, we prioritized using the authors' publicly available code and hyperparameter settings.
For LoRA and AdaLoRA, we set the rank to 8 and inserted adapters into the MLP modules of each layer.
For AlphaEdit-D, we assumed that the knowledge corresponding to the test questions could be correctly extracted as triplets, so we directly converted the test questions into triplets for editing.
For UnKE and AnyEdit, we directly used the given question provided in the dataset as constructed query for editing.
For FT, we selected the MLP module of the 21st layer for editing on both Llama and Qwen models, as it consistently yielded the best performance across all tested layers. We used the AdamW optimizer with a learning rate of 5e-4, a loss threshold of 1e-2, and a maximum of 25 training steps.
For \themodel, we kept other hyperparameter settings consistent with FT. Specifically, for pretrained models, we set the sliding window length and the weight of the context alignment loss ($k$ and $\alpha$) to 10 and 0.05, respectively; for instruction-tuned models, we set them to 5 and 0.1, respectively. For the weight of the Knowledge Consistency Loss ($\beta$), we set it to 0.5 for Llama models and 3e-4 for Qwen models. To compute the covariance matrix $K_0K_0^T$, we followed ROME \citep{rome} by sampling 100,000 key vectors from Wikipedia text.
All experiments were conducted on NVIDIA A800 (80G) GPUs.

\section{Further Analysis on \conr}

\subsection{\conr on UnKEBench}

\begin{figure*}
  \centering
  \includegraphics[width=0.9\linewidth]{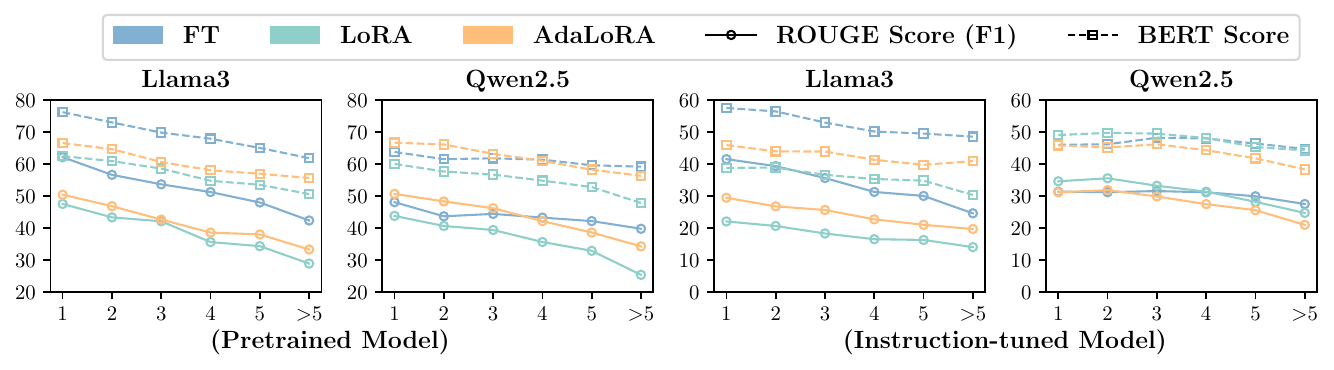}
  \vspace{-0.5em}
  \caption{Performance comparison of different methods on UnKEBench dataset. The x-axis represents the position of knowledge in the text, and the y-axis represents the corresponding accuracy. Results indicate that \conr is a prevalent issue in NTP-based unstructured knowledge editing.}
  \label{fig:unke_context_reliance}
  \vspace{-1.7em}
\end{figure*}

Figure \ref{fig:unke_context_reliance} illustrates the accuracy of edited models on questions targeting knowledge at varying positions within text on UnKEBench dataset.
Similar to the observations on AKEW dataset, there is a significant degradation in model performance when knowledge is located later in the text.
It is worth noting that, unlike AKEW dataset, testing cases in UnKEBench is not pre-sorted based on position of knowledge within text.
Therefore, we employed a LLM to sort the test data.
Given that the sorting by LLM may not be entirely accurate, the results on UnKEBench appear somewhat smoother compared to those on AKEW.
Despite this, the results still confirm that \conr issue is a prevalent problem in NTP-based unstructured knowledge editing, regardless of the specific dataset used.

\subsection{Impact of Model Scale}\label{appd:model_scale}

\begin{figure*}
  \centering
  \includegraphics[width=0.9\linewidth]{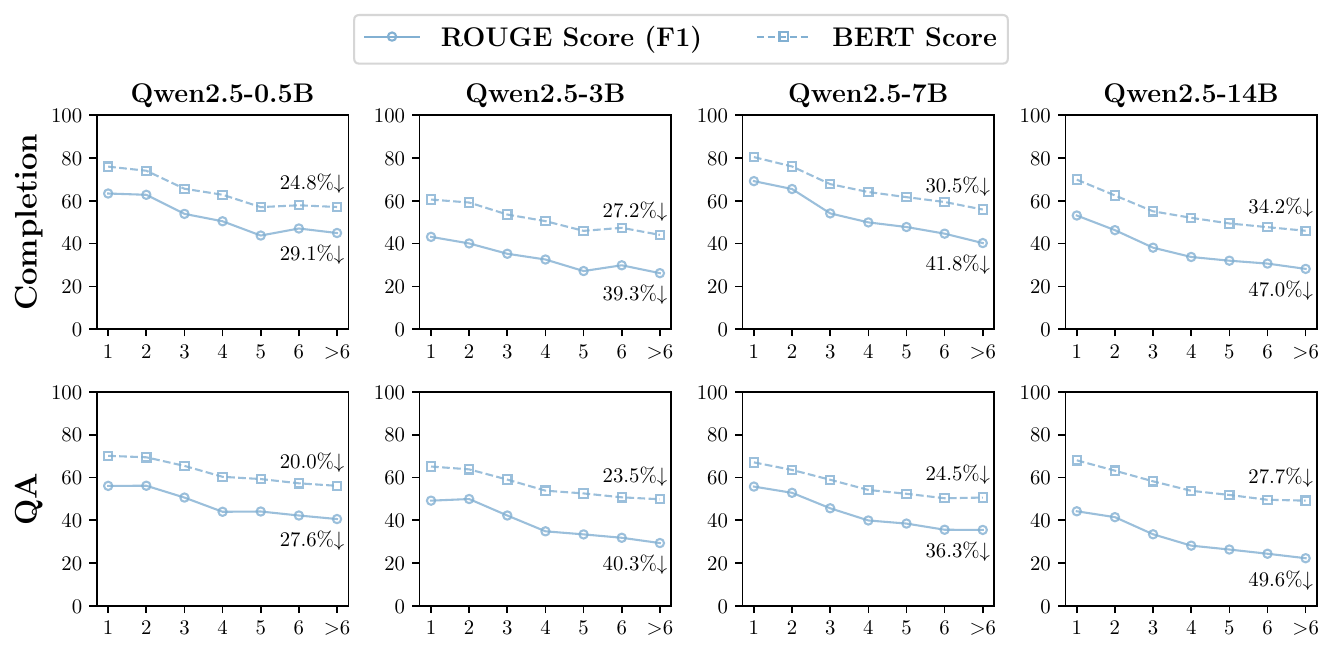}
  \vspace{-0.7em}
  \caption{\conr across different model scales. Results indicate that as model size increases, the \conr issue becomes more severe.}
  \vspace{-1.5em}
  \label{fig:model_scale}
\end{figure*}
To further validate the impact of model scale on \conr, we conducted experiments on instruction-tuned Qwen2.5 models of four different scales: 0.5B, 3B, 7B, and 14B.
We observed the model's performance in answering questions based on knowledge located at different positions using the FT method.
As shown in Figure \ref{fig:model_scale}, across all model scales, the accuracy of the model's answers significantly decreases as the position of testing questions shifts later, and the degree of decrease is more pronounced in larger models.
This indicates that the \conr issue is prevalent across models of different scales, and as the model size increases, the problem of \conr also intensifies, further highlighting the ubiquity and severity of the \conr problem.

\subsection{Impact of Training Steps}\label{appd:training_steps}
\begin{wrapfigure}{r}{0.4\linewidth}
  \vspace{-2em}
  \includegraphics[width=\linewidth]{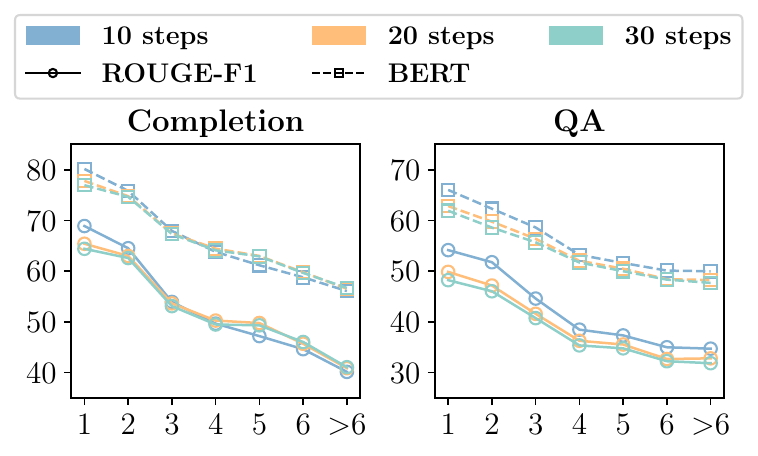}
  \vspace{-2em}
  \caption{Impact of training steps on \conr.}
  \label{fig:steps}
\end{wrapfigure}
To further validate the impact of training steps on the model's \conr, we conducted comparative experiments on the Llama3-8B-Instruct model with training steps set to 10, 20, and 30.
As shown in Figure \ref{fig:steps}, increasing the number of training steps does not reduce the model's reliance on context, as the accuracy of answers still significantly decreases as the position of knowledge moves further back.
Moreover, for QA-formatted testing cases, increasing the number of training steps can even severely affect the accuracy.
This may be because excessive training steps negatively impact the model's general capabilities, leading to some previously memorized knowledge being incorrectly outputted.
This experiment indicates that the \conr phenomenon doesn't arise from insufficient learning, but is more likely due to the inherent limitations of the NTP paradigm.

\section{Additional Experiments}
\subsection{Performance Comparison on Pretrained Models}
\begin{table}[t]
\setlength\tabcolsep{3.75pt}
\caption{Performance comparison of all methods on AKEW and UnKEBench using pretrained models, evaluated in terms of BERT Score (\%) and Precision, Recall, F1 for ROUGE-L (\%). The best results are in \textbf{bold}, and the second best are \underline{underlined}.}
\begin{center}
\begin{small}
\resizebox{0.9\linewidth}{!}{\begin{tabular}{cl|cccc|cccc|cccc}
\toprule
\multirow{2}{*}{\textbf{Model}} & \multirow{2}{*}{\textbf{Method}} & \multicolumn{4}{c}{\textbf{AKEW-Com.}} 
& \multicolumn{4}{c}{\textbf{AKEW-QA}} & \multicolumn{4}{c}{\textbf{UnKEBench}} \\
& & Prec. & Rec. & F1 & BERT & Prec. & Rec. & F1 & BERT & Prec. & Rec. & F1 & BERT \\
\midrule

\multirow{8}{*}{\rotatebox{90}{Llama3}}
& BASE
    & 25.80 & 20.95 & 21.83 & 41.68
    & 37.44 & 40.74 & 37.07 & 49.78
    & 23.36 & 24.64 & 21.73 & 46.53 \\

& AnyEdit
    & 25.97 & 21.35 & 22.04 & 42.05
    & 37.84 & 42.59 & 37.99 & 51.29
    & 15.86 & 30.13 & 18.46 & 37.46 \\

& UnKE
    & 26.06 & 21.44 & 22.19 & 41.79
    & 35.69 & 44.36 & 37.26 & 52.16
    & 19.71 & 26.93 & 20.48 & 43.47 \\

& AlphaEdit-D
	& 36.59 & 38.44 & 33.88 & 52.17
	& 38.82 & 48.62 & 39.31 & 53.04
	& 34.66 & 38.12 & 33.59 & 57.70 \\

& FT
    & \underline{58.79} & \underline{70.02} & \underline{59.35} & \underline{70.53}
    & \underline{46.51} & \underline{67.02} & \underline{50.18} & \underline{61.34}
    & \underline{50.66} & \underline{72.20} & \underline{54.55} & \underline{70.68} \\

& LoRA
    & 38.09 & 69.59 & 41.63 & 58.35
    & 34.25 & 59.78 & 39.20 & 55.44
    & 34.73 & 67.55 & 40.77 & 58.21 \\

& AdaLoRA
    & 40.97 & 63.29 & 42.79 & 59.45
    & 40.54 & 63.39 & 44.91 & 59.00
    & 38.19 & 67.08 & 43.45 & 61.59 \\

\rowcolor{gray!20} & \themodel
    & \textbf{61.20} & \textbf{76.51} & \textbf{63.33} & \textbf{73.59}
    & \textbf{47.80} & \textbf{69.69} & \textbf{51.92} & \textbf{61.73}
    & \textbf{52.17} & \textbf{74.83} & \textbf{56.63} & \textbf{72.32} \\

\rowcolor{gray!20} & $\Delta Improve$ & 4.1\% & 9.3\% & 6.7\% & 4.3\% & 2.8\% & 4.0\% & 3.5\% & 0.6\% & 3.0\% & 3.6\% & 3.8\% & 2.3\% \\

\midrule

\multirow{8}{*}{\rotatebox{90}{Qwen2.5}}
& BASE
    & 23.11 & 19.50 & 19.99 & 39.66
    & 37.73 & 41.65 & 37.71 & 51.06
    & 18.21 & 24.61 & 19.28 & 43.74 \\

& AnyEdit
    & 23.37 & 19.94 & 20.30 & 39.99
    & 38.24 & 43.73 & 38.84 & 52.64
    & 18.14 & 28.77 & 20.10 & 42.33 \\

& UnKE
    & 23.65 & 20.12 & 20.55 & 40.14
    & 37.76 & 45.79 & 39.17 & 53.76
    & 18.21 & 27.37 & 19.90 & 42.66 \\

& AlphaEdit-D
	& 40.35 & 40.40 & 37.06 & 56.69
	& 42.47 & 53.68 & 43.35 & 57.95
	& 39.49 & 46.88 & 40.31 & 61.32 \\

& FT
    & \underline{51.94} & \underline{73.96} & \underline{53.92} & \underline{67.07}
    & \underline{45.22} & \underline{71.59} & \underline{50.07} & \underline{62.46}
    & 37.41 & \underline{75.70} & 44.14 & 61.55 \\

& LoRA
    & 33.14 & 67.07 & 35.92 & 53.82
    & 35.54 & 63.17 & 39.63 & 56.23
    & 33.12 & 70.25 & 38.67 & 56.53 \\

& AdaLoRA
    & 45.50 & 59.58 & 45.60 & 61.59
    & 43.91 & 64.00 & 47.37 & 61.19
    & \underline{40.05} & 68.44 & \underline{45.61} & \underline{63.44} \\

\rowcolor{gray!20} & \themodel
    & \textbf{63.44} & \textbf{74.58} & \textbf{64.11} & \textbf{73.74}
    & \textbf{51.31} & \textbf{72.60} & \textbf{55.44} & \textbf{65.28}
    & \textbf{46.98} & \textbf{76.21} & \textbf{53.72} & \textbf{69.12} \\

\rowcolor{gray!20} & $\Delta Improve$ & 22.1\% & 0.8\% & 18.9\% & 9.9\% & 13.5\% & 1.4\% & 10.7\% & 4.5\% & 17.3\% & 0.7\% & 17.8\% & 9.0\% \\

\bottomrule
\end{tabular}}
\label{table:main_base}
\vspace{-1.5em}
\end{small}
\end{center}
\end{table}

Table \ref{table:main_base} presents the performance comparison of all methods on the AKEW and UnKEBench datasets using the pretrained Llama3-8B and Qwen2.5-7B models.
The experimental results demonstrate that \themodel outperforms all other methods on pretrained models, confirming its model-agnostic nature.
These results prove that aligning the distribution of knowledge across different context lengths during the editing process can effectively mitigate the \conr issue, thereby enhancing the effectiveness of unstructured knowledge editing.

\subsection{Multi-hop Editing Analysis}\label{appd:multi_hop}
\begin{wraptable}{r}{0.45\linewidth}
\setlength\tabcolsep{6.75pt}
\vspace{-1.5em}
\caption{Additional experimental results on \textsc{MQuAKE} dataset.}
\vspace{-1em}
\begin{center}
\begin{small}
\resizebox{\linewidth}{!}{\begin{tabular}{l|cc}
\toprule
\textbf{Method} & \textbf{Prob. (\%)} & \textbf{Match (\%)} \\
\midrule

FT (base) & 71.66 & 5.63 \\
FT (first layer) & 47.02 & 0.24 \\
FT (optimize target) & 71.27 & 0.43 \\

\bottomrule
\end{tabular}}
\label{table:mquake_analysis}
\end{small}
\end{center}
\vspace{-1.5em}
\end{wraptable}

To investigate the reasons behind FT's strong performance in multi-hop tasks, we conducted experiments with several variants of FT and evaluated them using additional metrics, as shown in Table \ref{table:mquake_analysis}.
Specifically, we designed two variants: one that fine-tunes the parameters of the first layer of the model instead of the 22nd layer as in original setting; and another that optimizes using only the gradients generated from the target, rather than utilizing gradients from the entire text.
In terms of metrics, in addition to comparing the probabilities of the new and old answers (Prob.), we also evaluate whether the outputs are completely identical to the new answers (Match).
The results indicate that fine-tuning the parameters of the first layer leads to a significant drop in performance, and using only the target's gradient for optimization also leads to decreased performance in probability difference and answer match.
This suggests that FT's superiority over traditional editing methods in multi-hop tasks stems from its use of gradient information from the complete text and its editing at deeper layers, thereby capturing the complex dependencies among multi-hop knowledge.

\subsection{Sensitivity Analysis}

\begin{wrapfigure}{r}{0.4\linewidth}
  \vspace{-1em}
  \includegraphics[width=\linewidth]{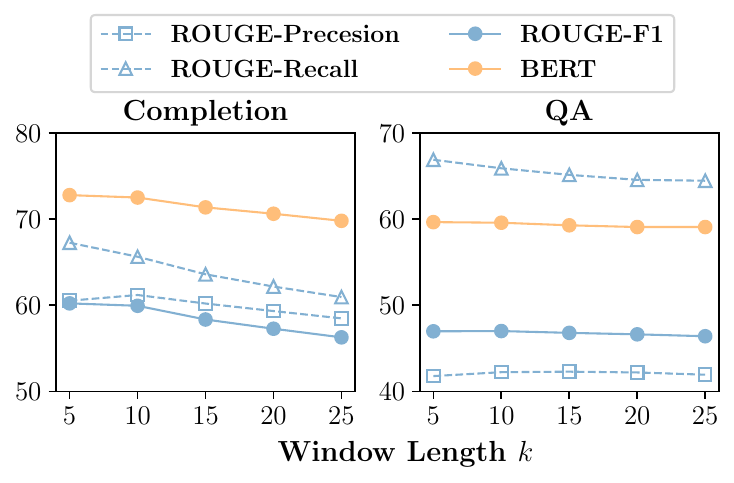}
  \vspace{-1.7em}
  \caption{Performance of \themodel with different window length $k$ on AKEW dataset.}
  \label{fig:sensitivity}
\end{wrapfigure}

To investigate \themodel's sensitivity to key hyperparameters, we evaluated the impact of the sliding window length $k$ on performance using the AKEW dataset. The parameter $k$ defines the length of the local context for our context alignment loss, a crucial component for effective alignment. We experimented with $k \in \{5, 10, 15, 20, 25\}$, with results presented in Figure \ref{fig:sensitivity}.
We observed that for both completion and QA tasks, the ROUGE-Recall score tended to decrease as $k$ increased. We hypothesize that as the local context window expands, it captures more information, reducing the distributional divergence between the local and global contexts. This, in turn, diminishes the effectiveness of the alignment loss.
Furthermore, we noted that performance on completion tasks was more sensitive to changes in $k$ than on QA tasks. This is likely because the completion task format more closely resembles the editing text, leading to a stronger dependence on the preceding text.

\subsection{Case Study}
In this section, we demonstrate the efficacy of \themodel to unstructured knowledge editing using the Llama3-8B-Instruct model. We present several examples from the AKEW and UnKEBench datasets, comparing the outputs with baseline methods including AnyEdit, AlphaEdit, FT, and LoRA. These case studies, illustrated in Figure \ref{fig:case_study_akew} and Figure \ref{fig:case_study_unke}, highlight the models' ability to answer questions based on the edited unstructured knowledge.

\textbf{AKEW Case Study.} Our first case study, from the AKEW dataset, involves the completion task. When presented with two questions targeting the edited knowledge, all baseline methods failed to produce the correct answer. In contrast, \themodel successfully answered the questions, providing a response identical to the ground truth. This result underscores the effectiveness of \themodel in accurately editing unstructured knowledge and responding to related queries.

\textbf{UnKEBench Case Study.} The second case study, utilizing the UnKEBench dataset, presents a text with multiple subquestions. Here, we focus on the outputs of the FT, LoRA and \themodel methods, which are based on the next-token prediction paradigm. We observed that the FT method tended to generate extra phrases, such as "According to his birth certificate," before providing the answer. This issue was even more prominent with LoRA, which, after generating the correct answer "Mart and Liisa Kull," continued to produce irrelevant content. We attribute this to the reliance on the preceding context during the editing process, which prevents these methods from focusing on the target knowledge. As a result, they appear to generate additional content to reason through the answer. In contrast, \themodel accurately answered all subquestions with responses that closely matched the ground truth. This further demonstrates that by mitigating the issue of \conr, \themodel can effectively edit unstructured knowledge and provide precise answers to related questions.

\definecolor{dgreen}{RGB}{0, 176, 80}
\begin{figure*}
  \centering
  \includegraphics[width=0.75\linewidth]{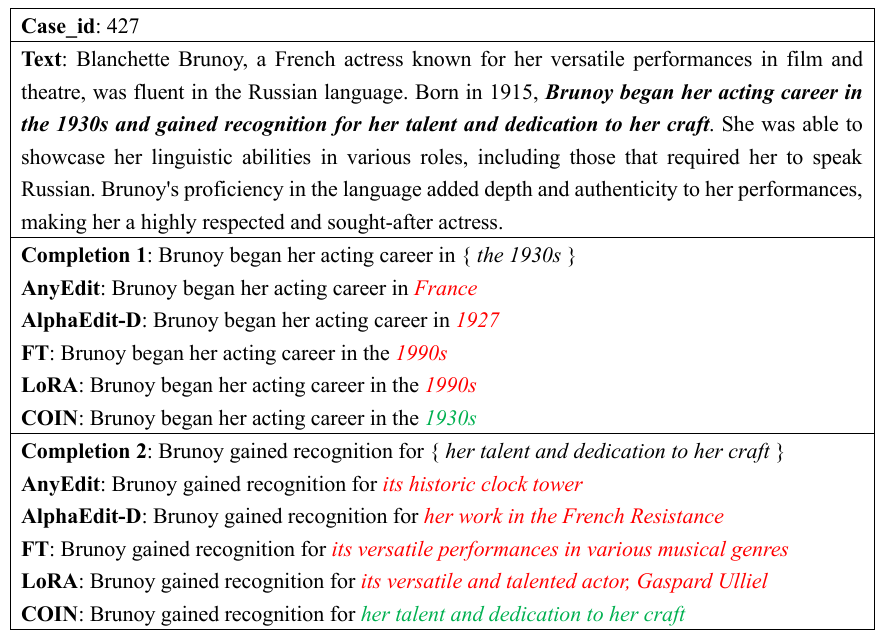}
  \caption{Comparison of outputs from Llama3-8B-Instruct on an example from the AKEW dataset, showcasing the performance of AnyEdit, AlphaEdit, FT, LoRA, and \themodel. \textcolor{dgreen}{\textbf{Green}} text indicates a correct answer, while \textcolor{red}{\textbf{red}} text highlights an incorrect one.}
  \label{fig:case_study_akew}
\end{figure*}

\begin{figure*}[t!]
  \centering
  \includegraphics[width=0.75\linewidth]{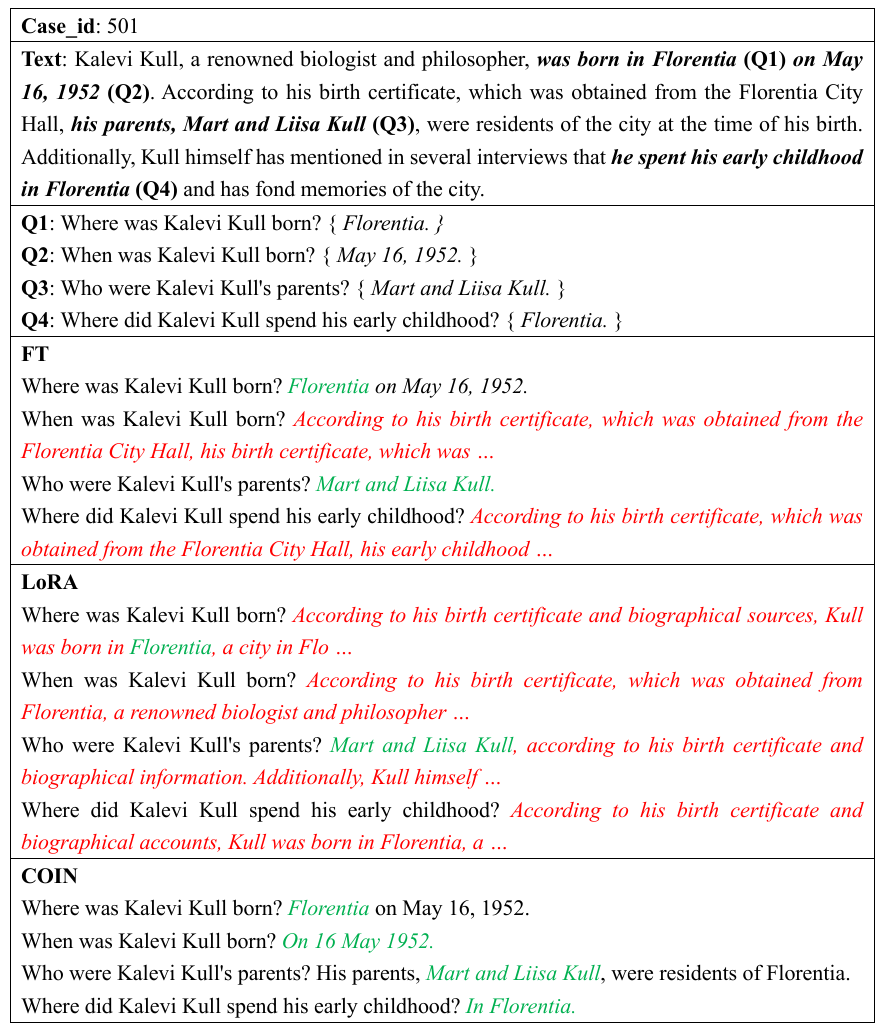}
  \caption{Comparison of outputs from Llama3-8B-Instruct on an example from the UnKEBench dataset, showcasing the performance of FT, LoRA, and \themodel. \textcolor{dgreen}{\textbf{Green}} text indicates correct answers, while \textcolor{red}{\textbf{red}} text highlights incorrect or irrelevant content.}
  \label{fig:case_study_unke}
\end{figure*}

\end{document}